\documentclass[10pt,twocolumn,letterpaper]{article}

\usepackage[pagenumbers]{cvpr} 

\usepackage{graphicx}
\usepackage{amsmath}
\usepackage{amssymb}
\usepackage{booktabs}

\usepackage{amsthm}
\usepackage{amsmath}
\usepackage{amssymb}
\usepackage{mathtools} 
 \usepackage{mathrsfs}
 \usepackage{hhline}
 \usepackage{algorithm}
\usepackage{algorithmic}
\usepackage{float}
\usepackage{stfloats}
\usepackage{footmisc}
\usepackage{arydshln}
\usepackage[vlined,linesnumbered,ruled,algo2e]{algorithm2e}
\usepackage[linesnumbered,ruled,vlined]{algorithm2e}
\usepackage[shortlabels]{enumitem}
\usepackage{custom_commands}
\usepackage[pagebackref,breaklinks,colorlinks]{hyperref}

\usepackage[capitalize]{cleveref}
\crefname{section}{Sec.}{Secs.}
\Crefname{section}{Section}{Sections}
\Crefname{table}{Table}{Tables}
\crefname{table}{Tab.}{Tabs.}

\def\cvprPaperID{1835} 
\def\confName{CVPR}
\def\confYear{2022}
\newcounter{daggerfootnote}

\begin{document}

\title{Global Convergence of MAML and Theory-Inspired Neural Architecture Search for Few-Shot Learning} 
\author{Haoxiang Wang\thanks{equal contribution} \quad Yite Wang\footnotemark[1] \quad Ruoyu Sun \quad Bo Li\\
University of Illinois Urbana-Champaign\\
{\tt\small \{hwang264,yitew2,ruoyus,lbo\}@illinois.edu}
}
\maketitle

\begin{abstract}
Model-agnostic meta-learning (MAML) and its variants have become popular approaches for few-shot learning. However, due to the non-convexity of deep neural nets (DNNs) and the bi-level formulation of MAML, the theoretical properties of MAML with DNNs remain largely unknown. In this paper, we first prove that MAML with over-parameterized DNNs is guaranteed to converge to global optima at a linear rate. Our convergence analysis indicates that MAML with over-parameterized DNNs is equivalent to kernel regression with a novel class of kernels, which we name as Meta Neural Tangent Kernels (MetaNTK). Then, we propose MetaNTK-NAS, a new training-free neural architecture search (NAS) method for few-shot learning that uses MetaNTK to rank and select architectures. Empirically, we compare our MetaNTK-NAS with previous NAS methods on two popular few-shot learning benchmarks, miniImageNet, and tieredImageNet. We show that the performance of MetaNTK-NAS is comparable or better than the state-of-the-art NAS method designed for few-shot learning while enjoying more than 100x speedup. We believe the efficiency of MetaNTK-NAS makes itself more practical for many real-world tasks. Our code is released at \url{github.com/YiteWang/MetaNTK-NAS}.
\end{abstract}
\vspace{-1em}
\section{Introduction}\label{sec:intro}
Meta-learning, or learning-to-learn (LTL) \cite{learningtolearn}, has received much attention due to its applicability in few-shot image classification \cite{few-shot-survey,hospedales2020metalearning}, meta reinforcement learning \cite{vanschoren2018meta,Finn:EECS-2018-105,hospedales2020metalearning}, and other domains such as natural language processing \cite{yu-etal-2018-diverse,bansal2019learning} and computational biology \cite{luo2019mitigating}. The primary motivation for meta-learning is to fast learn a new task from a small amount of data, with prior experience on similar but different tasks. There are a few meta learning approaches designed for few shot image classification, such as metric-based \cite{rnn1,snell2017prototypical}, model-based \cite{santoro2016meta,munkhdalai2017meta}, optimizer-based \cite{ravi2016optimization,li2016learning}.
Model-agnostic meta-learning (MAML) is a popular gradient-based 
meta-learning approach, due to its simplicity and good performance in many meta-learning tasks \cite{maml,Finn:EECS-2018-105}. 
MAML formulates a bi-level optimization problem, where the inner-level objective represents the adaption to a given task, and the outer-level objective is the meta-training loss.
There are many variants of MAML, \cite{reptile,adaptive-GBML,prob-maml,finn2019online,imaml},
and they are almost always applied together with deep neural networks (DNNs) in practice.

Even though MAML with DNNs is empirically successful, this approach still lacks a thorough theoretical understanding. For example, the most common practice is to use \textit{gradient descent} approach (e.g., SGD or Adam \cite{adam}) to train MAML with DNNs, and the optimization can usually obtain almost zero training loss and 100\% training accuracy (i.e., global convergence) with suitable hyper-parameters \cite{maml,howtotrainmaml}. However, prior theoretical works could not account for the global convergence of MAML trained with gradient descent on non-linear neural nets of more than two layers. Hence, a crucial question that remains unknown for MAML optimization is: 

\emph{Can MAML with DNNs converge to global minima?} 

This question motivates us to analyze the optimization properties of MAML with DNNs, and we provide a positive answer with rigorous theoretical analysis. Briefly, for over-parameterized DNNs, we analyze the optimization trajectory of MAML with square loss and prove that the training loss is guaranteed to converge to zero at a linear rate.
Additionally, in this convergence analysis, we find the DNN trained by MAML can be described by a kernel regression, with a novel class of kernels that we name as \textbf{\textit{Meta Neural Tangent Kernels}} (MetaNTK).

One may wonder whether our theory has any practical implications. 
Intuitively, our theory reveals that  MetaNTK is closely related to the performance of MAML.
To demonstrate the practical value of our theory, we provide a concrete use case of MetaNTK: MetaNTK can help us efficiently find neural net architecture for few-shot learning.

Most meta-learning algorithms adopt standard network structures such as ConvNets \cite{lecun2015deep}, ResNets \cite{resnet} and Wide ResNets \cite{wide-resnet} for few-shot image classification, the most popular task to benchmark meta-learning. However, these network structures were developed on supervised learning benchmarks such as CIFAR \cite{cifar}, and ImageNet \cite{imagenet}, and recently, it has been shown that these popular structures actually \textit{overfit} to the supervised learning task on these datasets \cite{recht2019imagenet}. This indicates that the popular network structures may not be optimal for tasks other than supervised learning, such as few-shot learning. Thus, one may naturally consider neural architecture search (NAS) \cite{NAS-RL,NAS-survey,DARTS} to automatically search for neural net architectures that are suitable for few-shot learning. To this end, prior works \cite{auto-meta,metaNAS,TNAS} designed NAS methods specific for few-shot learning, but they require substantial computational cost (e.g., the search cost of \cite{auto-meta} and \cite{metaNAS} on mini-ImageNet is 100 and 7 GPU days, respectively; the training of \cite{TNAS} takes 150 GPU days on miniImageNet), which makes them impractical for many real-world tasks and not environmental-friendly \cite{dhar2020carbon,xu2021survey}. Hence, a natural  question is:

\textit{Can we accelerate NAS for few-shot learning to have much lower or even negligible search cost (compared to training cost)?}

We provide an efficient solution to this quest, \textbf{\textit{MetaNTK-NAS}}, which is inspired by the MetaNTK we derive in our global convergence analysis of MAML. Briefly, we use the condition number of MetaNAK as an indicator for the trainability of networks under MAML. Since MetaNTK is directly computed at initialization, no training is needed in the search stage, leading to a surprisingly small search cost (e.g., less than 0.07 GPU day on mini-ImageNet).

Our main contributions are summarized below:
\vspace{-0.45em}
\begin{itemize}[leftmargin=*,align=left]
    \setlength\itemsep{0.01em}
    \item \textbf{Global Convergence and Induced Kernels of MAML\footnote{This part was also presented in a prior tech report of ours \cite{meta-ntk}.}}: 
    We prove that with over-parameterized DNNs (i.e., DNNs with a large number of neurons in each layer), MAML is guaranteed to \textit{converge to global optima with zero training loss} at a linear rate. The key to our proof is to develop bounds on the gradient of the MAML objective, and then analyze the optimization trajectory of DNN parameters trained under MAML. 
    Furthermore, we show that in the over-parameterization regime, the output of MAML-trained networks becomes equivalent to the output of a special kernel regression with a new class of kernels, Meta Neural Tangent Kernels (MetaNTK).
    \item \textbf{Theory-Inspired Efficient NAS for Few-Shot Learning:} 
We propose MetaNTK-NAS, a new NAS method for few-shot learning that takes advantage of MetaNTK. Briefly, it uses the condition number of the MetaNTK of each network as an indicator for its trainability under meta-learning. Empirically, our MetaNTK-NAS is comparable or better than MetaNAS \cite{metaNAS}, the state-of-the-art NAS method for few-shot learning, on both miniImageNet and tieredImageNet, while consuming 100x less cost in the search process.
\end{itemize}

\section{Related Works}

\textbf{Meta-Learning Optimization.} The MAML family (i.e., MAML \cite{maml} and its variants) is a popular approach for meta-learning. Several recent works theoretically analyze MAML or its variants in the case of \textit{convex} objectives \cite{finn2019online,provable-gbml,adaptive-GBML,hu2020biased,xu2020meta,bai2021important}.
However, neural nets are \textit{non-convex}, so these works do not account for common practices of MAML with neural nets.
On the other hand, \cite{maml_nonconvex,ji2020multistep,imaml,zhou2019metalearning} consider the non-convex setting, but they only provide convergence to \textit{stationary points}. Since \textit{stationary points} can have high training/test error, the convergence to them is not very satisfactory. In particular, \cite{wang2020global} proves the global convergence of MAML for a special class of \textit{two-layer} networks that has frozen last layers with binary weight values. The unrealistic setting of \cite{wang2020global} makes its results much weaker than our work, where our analysis is compatible with any depth and has no restriction on layer weight values. Notably, the MAML family shares similarities with multi-task learning from an optimization perspective \cite{wang2021bridging}.

\textbf{Neural Tangent Kernels.} Recently, there is a line of works studying the optimization of neural networks in the setting of supervised learning, e.g., \cite{AllenZhu2018ACT,du2018gradient,ntk,kawaguchi2019elimination,oymak2019moderate}, and many of them are restricted to two-layer networks. Notably, \cite{ntk} proves that gradient flow on infinitely wide neural nets of any depth is guaranteed to converge to global optima, while its training dynamics can be described by kernel regression with neural tangent kernels (NTK).
Further, \cite{lee2019wide} relaxes some assumptions of \cite{ntk}, and proves that gradient descent on finitely wide neural nets of any depth also enjoys global convergence as long as the width is large enough and the learning rate is sufficiently low. Notice that these works are tailored for supervised learning, thus it is unknown if global convergence is also guaranteed in other problems using neural networks. In this work, we leverage the tools of NTK from \cite{ntk,lee2019wide} to analyze MAML in the few-shot learning setting, and our analysis can be easily generalized to other variants of MAML such as \cite{raghu2019rapid,imaml,R2D2}.

\textbf{Neural Architecture Search.} Neural Architecture Search (NAS) is proposed to automate neural architecture discovery to replace cumbersome manual designs for various deep learning tasks. Early works successfully utilize reinforcement learning \cite{NAS-RL, baker2016designing} and evolutionary algorithm \cite{real2019aging} to find high-performance architectures. However, most of these methods are computationally expensive. To enable efficient architecture search, DARTS \cite{DARTS} proposed continuous relaxation of the architecture representation to allow search via gradient descent. Unfortunately, DARTS is hard to optimize and may suffer from performance degradation due to its weight-sharing strategy \cite{yu2019evaluating, wang2020rethinking}.
To further accelerate architecture search, \cite{mellor2021neural} proposed training-free NAS to evaluate randomly initialized architectures, thus fully eliminating neural network training during the search. Some following training-free methods propose to search with NTK \cite{KNAS, TE-NAS}, linear regions \cite{TE-NAS} and pruning-related criterion \cite{zeroproxy}.
On the other hand, there are a few works applying NAS to few-shot learning using some meta-learning algorithms. \cite{auto-meta} apply progressive neural architecture search to few-shot learning and \cite{TNAS, metaNAS} adopt DARTS-variants. But these approaches are very costly (e.g., \cite{auto-meta,TNAS} take more than 100 GPU days and \cite{metaNAS} takes over 1 GPU week), thus it remains unexplored how to \textit{efficiently} apply NAS to few-shot learning.

\section{Preliminaries}

\paragraph{Few-Shot Learning} Consider a few-shot learning problem with a set of \textit{training tasks} that contains $N$ \textit{supervised-learning tasks} $\{\task_i\}_{i=1}^{N}$. Each task is represented as
$$\task_i= (\xyxyi)\in \mathbb R^{n\times d} \times \mathbb R^{n  k} \times \mathbb R^{m\times d} \times \mathbb R^{m  k},$$ where $(\sX_i,\sY_i)$ represents
$n$ \textit{query} samples (i.e. test samples of $\task_i$) with corresponding labels, while $(\sX_i',\sY_i')$ represents $m$ \textit{support} samples (i.e. training samples of $\task_i$) with labels. Then, we denote $$\X=(\sX_i)_{i=1}^N,\Y=(\sY_i)_{i=1}^N,\X'=(\sX_i')_{i=1}^N,\Y'=(\sY_i')_{i=1}^N.$$
In few-shot learning, $\{\task_i\}_{i=1}^N$ are training tasks for meta-learners to train on (i.e., for meta-training). 
In the evaluation stage, an arbitrary test task $\task = (\xyxy)$ is picked, and the labeled support samples $(X',Y')$ are given to the trained meta-learner as input, then the meta-learner is asked to predict the labels of the query samples $X$ from $\task$.
Notice that this few-shot learning problem above can also be called a $n$-shot $k$-way learning problem. 
\vspace{-1em}
\paragraph{Neural Net Setup} Consider a neural network $f_\theta$ with $L$ hidden layers, where parameters $\theta \in \mathbb R^D$.
For $i\in[L]$, we use $l_i$ to denote the width of the $i$-th hidden layer. Without loss of generality, we consider all hidden layers have the same width\footnote{This same-width assumption is not a necessary condition. One can also define $l=\min{\{l_i\}_{i=1}^L}$ instead and all theoretical results in this paper still hold true.}, i.e., $l_1=l_2=\dots=l_L=l$. We consider the parameters $\theta$ are Gaussian initialized\footnote{Kaiming initialization \cite{resnet}is also a kind of Gaussian initialization.}, with details shown in Appendix \ref{supp:NTK-setup}.
\vspace{-1em}
\paragraph{MAML} The algorithm of MAML is shown in Algorithm \ref{alg:mamlsup}. For simplicity, it shows MAML with one inner-loop update step, while our theory is compatible with arbitrary steps of inner-loop update (e.g., Line 5 of Algorithm \ref{alg:mamlsup} can be modified to have multiple gradient descent steps). For convenience, we define a \textit{meta-output} function $F$ as the output of the model $f$ with adapted parameters. For MAML with one inner-loop step, the meta-output on arbitrary task $\task = (X,Y,X',Y')$ is
\begin{align}
    F_{\theta}(\xxy) = f_{\theta'}(X), \text{ s.t. } \theta' = \theta - \nabla_{\theta} \ell(f_\theta(X'),Y') \label{eq:meta-output-1-step}
\end{align}
where $f_{\theta'}(X) = (f_{\theta'}(x))_{x\in X}$ is the concatenation of model outputs on all samples in $X$. 

In this paper, we consider the square loss function $\ell(\hat y,y ) = \frac 1 2 \|\hat y - y\|_2^2$ for convenience. Then the MAML training loss is
\vspace{-.8em}
\begin{align}
    \loss(\theta) &= \frac{1}{2}\sum_{i\in[N]}\left\|F_{\theta}(\xxyi)-Y_i\right\|_2^2\nonumber\\ 
    &= \frac 1 2 \|F_{\theta}(\XXY) - \Y\|_2^2 \label{eq:MAML-obj}
\vspace{-1em}
\end{align}
where $F_\theta(\XXY) = \left(F_\theta(\xxyi)\right)_{i=1}^N $ is the 
concatenation of meta-outputs.

\vspace{-1em}
\paragraph{Second-Order Gradient (Hessian) in MAML} It is well known that gradient descent on the MAML objective \eqref{eq:MAML-obj} induces Hessian terms. For example, for MAML with one inner-loop step such as \eqref{eq:meta-output-1-step}, the gradient of objective \eqref{eq:MAML-obj} is
\begin{align}
    \nabla_\theta \loss (\theta) &= \sum_{i\in[N]} \nabla_{\theta}F_{\theta}(\xxyi)\left(F_{\theta}(\xxyi)-Y_i\right)\nonumber\\
    & = \sum_{i\in[N]} \left(I - \nabla_\theta^2 \ell(f_\theta(X'),Y')\right) \nabla_{\theta'} f_{\theta'}(X_i) \nonumber\\
    &\qquad \qquad  \cdot \left(F_{\theta}(\xxyi)-Y_i\right)\label{eq:hessian-term-demo} 
\end{align}
where $\nabla_\theta^2 \ell(f_\theta(X'),Y') \in \bR^{D\times D}$ is a Hessian term. Modern neural nets typically have millions of parameters, e.g., $D > 10^{7}$ in ResNet-12 \cite{resnet} that is commonly used in recent few-shot learning works \cite{metaOptNet,tian2020rethink}. Thus the Hessian terms usually have huge computation and memory cost.

\begin{algorithm}[t!]
\caption{MAML for Few-Shot Learning (version of one inner-loop step)}
\label{alg:mamlsup}
\begin{algorithmic}[1]
{\small
\REQUIRE $\{\task_i\}_{i=1}^N$: Training Tasks
\REQUIRE $\eta$, $\lambda$: Learning rate hyperparameters
\STATE Randomly initialize $\theta$
\WHILE{not done}
  \FORALL{$\task_i$} 
    \STATE Evaluate the loss of $f_\theta$ on support samples of $\task_i$: $ \ell(f_\theta(X_i'),Y_i')$
    \STATE Compute adapted parameters $\theta_i'$ with gradient descent: $\theta_i'=\theta-\lambda \nabla_\theta \ell(f_\theta(X_i'),Y_i')$
    \STATE Evaluate the loss of $f_{\theta_i'}$ on query samples of $\task_i$: $\ell(f_{\theta_i'}(X_i), Y_i)$
 \ENDFOR
 \STATE Update parameters with gradient descent:\\
    $\theta \leftarrow \theta - \eta \nabla_\theta \sum_{i = 1}^N  \ell ( f_{\theta_i'}(X_i),Y_i)$ 
\ENDWHILE
}
\end{algorithmic}
\end{algorithm}

\section{Theoretical Results}
\textbf{Notation.} For notational convenience, we denote the Jacobian of the meta-output on training data as $J(\theta)=\nabla_\theta F_{\theta}(\XXY)$, and define a kernel function, 
\begin{align}\label{eq:def-empirical-metantk}
\metantk_{\theta}(\cdot, \ast)\coloneqq \frac{1}{l}\nabla_\theta F_\theta(\cdot) \nabla_\theta F_\theta(\ast)^\top ~,
\end{align}
which we name as \textit{Meta Neural Tangent Kernel} function (MetaNTK). As the width $l$ approaches infinity, for Gaussian randomly initialized parameters $\theta_0$, the kernel function $\metantk_{\theta_0}(\cdot, \ast)$ becomes a deterministic function independent of $\theta_0$ (proved by Lemma \ref{lemma:kernel-convergence} in Appendix \ref{supp:global-convergence}), denoted as
\begin{align}
\metaNTK(\cdot, \ast) \coloneqq \lim_{l \rightarrow \infty} \metantk_{\theta_0}(\cdot, \ast) \label{eq:def-metaNTK-inf-lim}
\end{align}
For convenience, we denote $F_t(\cdot) \triangleq F_{\theta_t}(\cdot)$, $f_t(\cdot) \triangleq f_{\theta_t}(\cdot)$ and $\metantk_t(\cdot,\ast) \triangleq \metantk_{\theta_t}(\cdot,\ast)$. 
For any diagonalizable matrix $M$, we use $\lev(M)$ and $\Lev(M)$ to denote its least and largest eigenvalues.

\subsection{Global Convergence of MAML}\label{sec:global-convergence}

Suppose the neural network is sufficiently over-parameterized, i.e., the width of hidden layers, $l$, is large enough. Then, we prove that gradient descent on the MAML objective \eqref{eq:MAML-obj} is guaranteed to converge to global optima with zero training loss at a linear rate. The detailed setup, assumptions, and proof can be found in Appendix \ref{supp:global-convergence}. We provide a simplified theorem below.
\begin{theorem}[Global Convergence]\label{thm:global-convergence}
Define $\metaNTK = \lim_{l\rightarrow \infty} \frac{1}{l}J(\theta_0)J(\theta_0)^T$ and $\eta_0=\frac{2}{\Lev(\metaNTK) + \lev(\metaNTK)}$. For arbitrarily small $\delta > 0$, and there exist $R>0$, $l^* \in\mathbb N$, and $\lambda_0>0$, such that: for width $l\geq l^*$, running gradient descent with learning rates $\eta < \frac{\eta_0}{l}  $ and $\lambda < \frac{\lambda_0}{l}$ over random initialization, the following upper bound on the training loss holds true with probability at least $1 - \delta$:
\begin{align}
    \loss(\theta_t) &= \frac{1}{2}\|F_{\theta_t}(\XXY) - \Y \|_2^2 \nonumber\\
    &\leq \left(1 - \frac {\eta_0 \lev(\metaNTK)}{3}\right)^{2t} R \, .  
    \label{eq:convergence-loss}
\end{align}
\end{theorem}

\paragraph{Main Proof Ideas} Here, we depict the big picture of our global convergence proof for MAML with DNNs. To prove the global convergence of MAML with DNNs, we first obtain Lemma \ref{lemma:local-Liphschitzness} (shown in Appendix \ref{supp:global-convergence}), which indicates that the Jacobian of the meta-output, $J$, changes locally in a small region under perturbations on initial network parameters. Then, we analyze the training dynamics of MAML and show that the parameter movement during training is confined in a small region. Hence, the Jacobian is stable across training, indicating the loss landscape is almost quadratic in the local neighborhood of initialization. Importantly, we find that for sufficiently large width, there definitely exists a global minimum with zero training loss inside the neighborhood of any parameter initialization $\theta_0$. As a result, the almost quadratic loss landscape guarantees that gradient descent with a sufficiently small learning rate will reach this global minimum in the neighborhood at a linear rate.

\paragraph{Challenges} Even though the big picture of our global convergence proof for MAML may look intuitive and simple, there exist several severe challenges in the analysis that do not appear in the supervised learning setting. Here we give one example of the challenges we deal with in the analysis: MAML is a bi-level optimization problem, and each (outer-loop) gradient descent step on its training objective \eqref{eq:MAML-obj} consists of (multiple) inner-loop gradient descent steps. Hence, the Jacobian (i.e., the gradient of the meta-output) consists of Hessian (i.e., second-order gradient) terms that do not exist in supervised learning models, and these Hessian terms appear in the form of matrix exponentials. In order to prove the local stability of the Jacobian (i.e., Lemma \ref{lemma:local-Liphschitzness}), we utilize non-trivial theoretical tools on matrix exponentials \cite{1977bounds,van1977sensitivity}
that can tackle this challenge.

\subsection{Equivalence between MAML and Kernels}
\textbf{Analytical Expression of Meta-Output.} In the setting of Theorem \ref{thm:global-convergence}, the training dynamics of the MAML can be described by a differential equation
\begin{align}\label{eq:train-dynamics-ODE}
\frac{d F_t(\XXY)}{d t}=- \eta  \, \metantk_0   (F_t(\XXY) - \Y)
\end{align}
where we denote $\metantk_0 \equiv \metantk_{\theta_0}((\XXY),(\XXY))$ and $F_t \equiv F_{\theta_t}$ for convenience. Notice that \eqref{eq:train-dynamics-ODE} is a first-order ODE. The solution of this ODE gives rise to the analytical expression of $F_t$ on arbitrary task $\task=(\xyxy)$,
\begin{align}  \label{eq:F_t:main_text}
F_t(\xxy) &=F_{0}(\xxy)\\
 +&\metantk_0(\xxy) \T^{\eta}_{\metantk_0}(t)\left(\Y-F_0(\XXY)\right)\nonumber
\end{align}
where $\metantk_0(\cdot)\equiv \metantk_{\theta_0}(\cdot,(\XXY))$ and $\T^{\eta}_{\metantk_0}(t)=\metantk_0^{-1}\left(I- e^{-\eta\metantk_0 t}\right)$ are shorthand notations.

Here, we provide a formal definition of NTK \cite{ntk,lee2019wide}, which will be used shortly.
\begin{definition}[NTK] For any neural net function $f: \bR^{k} \mapsto \bR$ with randomly initialized parameters $\theta_0$, its neural tangent kernel function is defined as $\ntk_{0}(\cdot, \ast) \equiv \ntk_{\theta_0}(\cdot, \ast) \coloneqq \nabla_{\theta_0}f_{\theta_0}(\cdot) \nabla_{\theta_0}f_{\theta_0}(\ast)^\top $, and it converges to a deterministic kernel as the width $l$ approaches infinity \cite{ntk,lee2019wide},
\begin{align}
    \NTK(\cdot,\ast) = \lim_{l\rightarrow \infty} \ntk_{0}
\end{align}
\end{definition}
Below, we provide an analytical expression of Meta Neural Tangent Kernel (MetaNTK) for MAML in the infinite width limit, indicating that the kernel function $\metaNTK$ defined in \eqref{eq:def-metaNTK-inf-lim} can be equivalently viewed as a composite kernel function built upon NTK function $\NTK(\cdot, \ast)$. The derivation of this expression can be found in Lemma \ref{lemma:kernel-convergence} in Appendix \ref{supp:global-convergence}.
\begin{definition}[MetaNTK in the Infinite Width Limit]\label{def:metaNTK}
As the width $l$ approaches infinity, MetaNTK can be expressed as $\metaNTK \equiv \metaNTK((\XX),(\XX)) \in \mathbb{R}^{knN \times knN}$, which is a block matrix that consists of $N \times N$ blocks of size $kn\times kn$. For $i,j \in [N]$, its $(i,j)$-th block is
\begin{align} \label{eq:MetaNTK_ij=kernel}
    [\metaNTK]_{ij}=\SingleTaskmetaNTK((\xxi),(\xxj)) \in \mathbb R^{kn \times kn} , 
\end{align}
where $\SingleTaskmetaNTK: (\mathbb{R}^{n \times k} \times \mathbb{R}^{m\times k}) \times  (\mathbb{R}^{n \times k} \times \mathbb{R}^{m\times k}) \rightarrow \mathbb{R}^{nk \times nk}$ is a kernel function defined as
\begin{align}\label{eq:MetaNTK_ij}
    &\qquad \SingleTaskmetaNTK((\cdot,\ast), (\bullet, \star)) \\
    &= \NTK(\cdot,\bullet) + \NTK(\cdot,\ast)\widetilde{\T}_{\NTK}^\lambda(\ast,\tau)\NTK(\ast,\star)\widetilde{\T}_{\NTK}^\lambda(\star,\tau)^\top \NTK(\star,\bullet) \nonumber\\
 &~ -\NTK(\cdot,\ast)\widetilde{\T}_{\NTK}^\lambda(\ast,\tau) \NTK(\ast,\bullet) - \NTK(\cdot,\star) \widetilde{\T}_{\NTK}^\lambda(\star,\tau)^\top \NTK(\star,\bullet) .\nonumber
    \end{align}
\end{definition}

The following theorem shows that as the width of neural nets approaches infinity, MAML becomes equivalent to a special kernel regression with the MetaNTK of \cref{def:metaNTK}. The proof is provided in Appendix \ref{supp:MetaNTK}.
\begin{theorem}[MAML as Kernel Regression]\label{thm:MetaNTK}
Suppose learning rates $\eta$ and $\lambda$ are infinitesimal. As the network width $l$ approaches infinity, with high probability over random initialization of the neural net, the MAML output, (\ref{eq:F_t:main_text}), converges to a special kernel regression,
\begin{align}\label{eq:F_t-MetaNTK}
&F_t(\xxy) = G_\NTK^{\tau}(\xxy) \\
&\quad +\metaNTK((\xx),(\XX)) \T^{\eta}_{\metaNTK}(t) \left(\Y-G_{\NTK}^{\tau}(\XXY)\right)\nonumber
\end{align}
where $G$ is a function defined below.
\begin{align}\label{eq:G}
    G_\NTK^\tau(\xxy) =  \NTK(X,X')\widetilde{T}^{\lambda}_\NTK (X',\tau)   Y'.
\end{align}
with $\widetilde{T}^{\lambda}_\NTK (\cdot,\tau) \triangleq \NTK(\cdot,\cdot)^{-1}(I-e^{-\lambda \NTK(\cdot,\cdot) \tau}) $. Besides, $G_\NTK^{\tau}(\XXY) = (G_\NTK^{\tau}(\xxyi))_{i=1}^N$.

The $\metaNTK((\xx),(\XX)) \in \mathbb{R}^{kn \times knN}$ in (\ref{eq:F_t-MetaNTK}) is also a block matrix, which consists of $1 \times N$ blocks of size $k n \times k n$, with the $(1,j)$-th block as follows for $j \in [N]$, $$[\metaNTK((\xx),(\XX))]_{1,j}= \SingleTaskmetaNTK((X,X'),(X_j,X_j')).$$ 
\end{theorem} 

\textbf{Remarks.} The kernel $\metaNTK$ is what $\metantk_{0}$ converges to as the network width approaches infinity. However, $\metantk_{0} \equiv \metantk_0((\XXY),(\XXY))$ depends on $\Y$ and $\Y'$, while $\metaNTK \equiv \metaNTK((\XX),(\XX))$ does not, since the terms in $\metaNTK$ depending on $\Y$ or $\Y'$ all vanish as the width approaches infinity. Although (\ref{eq:F_t-MetaNTK}) is a sum of two kernel regression terms, it can still be viewed as a single special kernel regression relying on MetaNTK $\metaNTK$.  Notably, $\metaNTK$ can be seen as a \textit{composite} kernel built upon the \textit{base} kernel function $\NTK$. 

\section{MetaNTK-NAS}\label{sec:algo}

\begin{algorithm2e}[t!]
\small
	\textbf{Input:} Search step $k = 0$. A initial supernet $\mathcal{N}_0$ of stacked cells. Each cell has $E$ edges. Each edge has $|\mathcal{O}|$ operators.  \\
	\While {$\mathcal{N}_k$ is not a single-path network}{
	    \For {each operator $o_j$ in $\mathcal{N}_k$} {
	        $\Delta C_{k, o_j} = C_{\mathcal{N}_k} - C_{\mathcal{N}_k \backslash o_j}$\\ (The \underline{higher} $\Delta C_{t, o_j}$, the more likely we will prune $o_j$)\\
	        $\Delta R_{t, o_j} = R_{\mathcal{N}_k} - R_{\mathcal{N}_k \backslash o_j}$ \\
	        (The \underline{lower} $\Delta R_{t, o_j}$, the more likely we will prune $o_j$)\\
	    }
	    Get importance score of $C_{\mathcal{N}}$: $s_C(o_j) = $ index of $o_j$ in \underline{descendingly} sorted list $[\Delta C_{t, o_1}, ..., \Delta C_{t, o_{|\mathcal{N}_k|}}]$\\
	    Get importance score of $R_{\mathcal{N}}$: $s_R(o_j) = $ index of $o_j$ in \underline{ascendingly} sorted list $[\Delta R_{t, o_1}, ..., \Delta R_{t, o_{|\mathcal{N}_k|}}]$\\
	    Get total importance score $s(o_j) = s_C(o_j) + s_R(o_j)$\\
	    $\mathcal{N}_{k+1} \gets \mathcal{N}_k$\\
	    \For {each edge $e_i$, $i=1,...,E$} {
	        $j^* = \argmin_j \{s(o_j): o_j \in e_i\}$\\
	        (Find the operator with the \underline{least importance} score on each edge.)\\
	        $\mathcal{N}_{k+1} = \mathcal{N}_{k+1} \backslash o_{j^*}$
	    }
	    $k \gets k + 1$
	}
	\Return{Single-path network $\mathcal{N}_k$.}
	\caption{MetaNTK-NAS: Training-Free NAS for Few-Shot Learning.}
	\label{algo:metaNTK-NAS}
\end{algorithm2e}

Recently, there has been a series of \textit{training-free NAS} works that reduce the search cost of NAS by directly searching over untrained candidate structures \cite{mellor2021neural,TE-NAS,xu2021knas,ming_zennas_iccv2021}. The key of these works is to measure certain properties or metrics of untrained networks that are correlated with the final training/test accuracy. Among them, TE-NAS \cite{TE-NAS} utilizes two theory-motivated metrics of untrained networks to achieve training-free NAS: (\textit{i}) the condition number of the neural tangent kernel (NTK), which can be seen as an indicator for the trainability of the network under supervised learning;  (\textit{ii}) the number of linear regions in the input space, which implies the expressivity (i.e., representation power) of the network. Empirical studies of \cite{TE-NAS} confirm that these two metrics of untrained networks are correlated with the test accuracy of trained networks with the same structures.

Since few-shot learning is quite different from supervised learning, metrics specific to supervised learning may not be suitable for few-shot learning. For example, NTK is derived in the supervised learning setting, thus the use of condition number of NTK in TE-NAS \cite{TE-NAS} might be ineffective in few-shot learning. 
A natural candidate for few-shot learning is the MetaNTK we derive in \cref{thm:MetaNTK}, which can be viewed as the counterpart of NTK in meta-learning, thus it is a natural idea to replace NTK with MetaNTK in TE-NAS to accommodate the few-shot learning task, which is the core of our proposed MetaNTK-NAS.

Briefly speaking, the main idea of our MetaNTK-NAS is to compute an importance score for each untrained candidate structure by the condition number of MetaNTK and the number of linear regions in the input space, then search for a structure that maximizes this score. 
\begin{figure}[t]
  \centering
   \includegraphics[width=0.9\linewidth]{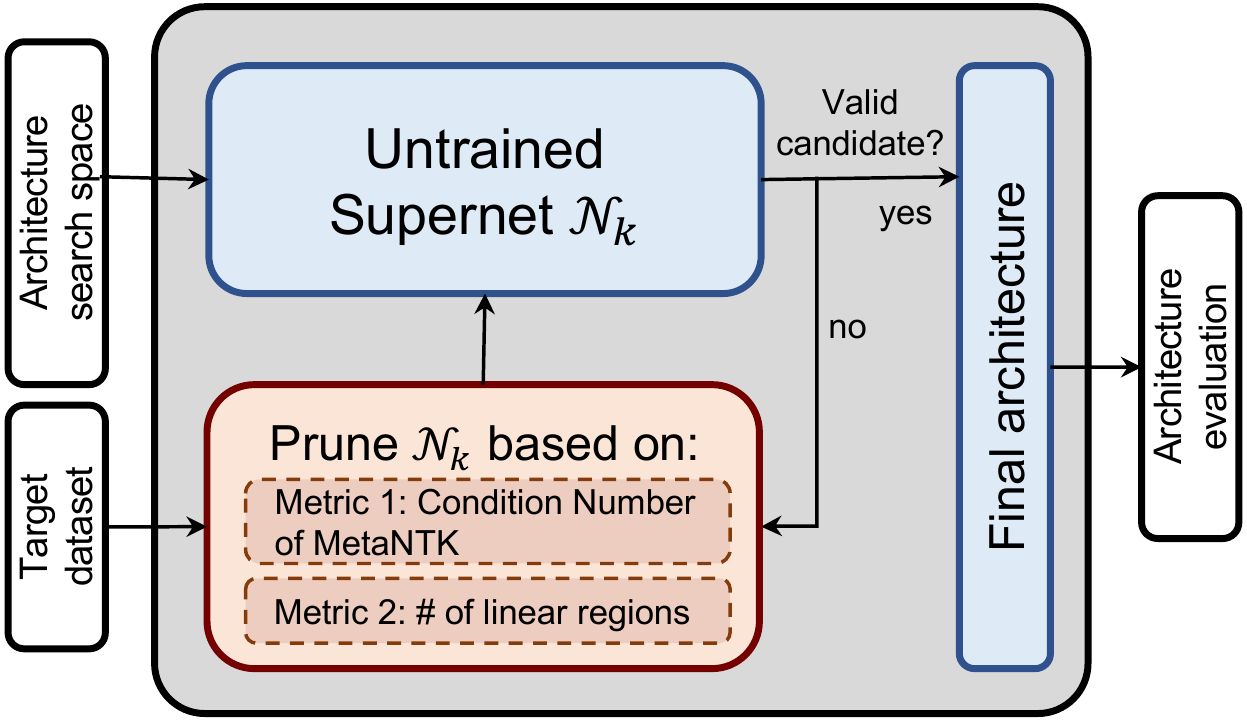}
   \caption{Illustration of our MetaNTK-NAS (\cref{algo:metaNTK-NAS}). }
   \label{fig:illustration}
   \vspace{-1em}
\end{figure}

Specifically, for each candidate structure $\mathcal{N}$ with parameters $\theta_0$ randomly initialized by Kaiming Initialization \cite{resnet}, we sample a batch of training tasks $\{(\xyxyi)\}$ from the training set, where $X_i\in \bR^{n\times k},X_i'\in \bR^{m\times k}$ are $n$ query samples and $m$ support samples, respectively. Then, we compute the MetaNTK of the network based on the analytical expression derived in \cref{eq:MetaNTK_ij=kernel,eq:MetaNTK_ij}. Notice that Eq. \eqref{eq:MetaNTK_ij} demonstrates the MetaNTK $\metaNTK$ is a composite kernel built upon the NTK kernel function in the infinite width limit. We assume this holds approximately true for finitely wide networks, and use the formula \eqref{eq:MetaNTK_ij} to build MetaNTK $\metantk_{\theta_0}$ from NTK kernel function $\ntk_{\theta_0}$. This approach has a computational benefit: if we compute MetaNTK following the definition in \eqref{eq:def-empirical-metantk}, that will involve second-order gradients; however, building MetaNTK from NTK following \eqref{eq:MetaNTK_ij} could get rid of this burden since $\ntk_{\theta_0}$ is computed purely from first-order gradients. Also, $\metantk_{\theta_0}$ is positive definite\footnote{Since NTK is positive definite \cite{ntk} and we build MetaNTK from NTK in a way that preserves the positive definiteness.}, thus its eigenvalues are all positive, and we can define its condition number to be
\begin{align}
    C_{\mathcal N} \coloneqq \frac{\sigma_{\max} (\metaNTK_{\theta_0})}{\sigma_{\min}( \metaNTK_{\theta_0})}
\end{align}
We compute another metric following Definition 1 of \cite{TE-NAS},
\begin{align}
    R_{\mathcal N} \coloneqq \text{number of linear regions}
\end{align}
The condition number $C_{\mathcal N}$ indicates the trainability\footnote{The connection between the condition number of NTK and trainability is discussed in Sec. 3.1.1 of \cite{TE-NAS}} of the network structure (the lower, the better) under meta-learning, and $R_{\mathcal N}$ stands for the representation power of the structure (the higher, the better). Our \cref{algo:metaNTK-NAS} is designed to minimize $C_{\mathcal N}$ and maximize $R_{\mathcal N}$ over candidate structures in the search space, following the algorithm design of TE-NAS \cite{TE-NAS}. An illustration of \cref{algo:metaNTK-NAS} is provided in \cref{fig:illustration}.

\section{Experiments}\label{sec:exp}

\subsection{Experiment Setup}
We conduct experiments on two popular few-shot image classification datasets, mini-ImageNet and tiered-ImageNet, which both are subsets of ImageNet \cite{imagenet}. And our experiments consist of three stages: search, train and evaluate. Here we give an overview of the datasets.
\vspace{-.1em}
\begin{itemize}[leftmargin=*,align=left]
\item \textit{mini-ImageNet} \cite{matching-net}: It contains 60,000 RGB images of 84x84 pixels extracted from ImageNet \cite{imagenet}. It includes 100 classes (each with 600 images) that are split into 64 training classes, 16 validation classes and 20 test classes.

\vspace{-1.0em}
\item \textit{tiered-ImageNet} \cite{ren2018metalearning}: This dataset contains 779,165 RGB images of 84x84 pixels extracted from ImageNet \cite{imagenet}. It includes 608 classes that are split into 351 training, 97 validation and 160 test classes.
\end{itemize}
\vspace{-2em}
\paragraph{Search Space} Following the NAS literature \cite{DARTS, TE-NAS, metaNAS}, we search for a normal cell and a reduction cell as the building block of the final architecture. Both cells have three intermediate nodes. MetaNAS \cite{metaNAS} uses a modified version of standard DARTS search space. We use the same search space as MetaNAS for a fair comparison. Specifically, the set of candidate operations include \textit{MaxPool3x3, AvgPool3x3, SkipConnect, Conv1x5-5x1, Conv3x3, SepConv3x3} and \textit{DilatedConv3x3}.
\vspace{-1.2em}
\paragraph{Implementation Details}The neural network we use is obtained by stacking 5 or 8 searched cells together. Cells located at the 1/3 and 2/3 of the total depth of the network\footnote{The network has an output layer in addition to cells, e.g., a 5-cells network has $5+1=6$ layers in total.} are reduction cells, where we decrease the spatial resolution and double the number of channels\footnote{This is different from the original setup of MetaNAS \cite{metaNAS}, where they fix number of channels throughout the whole network.}. We set the initial number of channels to 48\footnote{We use the same number of initial channels and number of cells during search and evaluation when computing MetaNTK, which is different from TE-NAS \cite{TE-NAS}. TE-NAS uses a small proxy network to compute NTK. 
}. To search for the candidate cells, we start with a supernetwork $\mathcal{N}_0$ composed of all possible edges and operations. We follow TE-NAS \cite{TE-NAS} to prune one operator on each edge by its importance per round. The importance of each operator is measured by the change of condition number of MetaNTK $C_{\mathcal{N}}$ and the number of linear regions $R_{\mathcal{N}}$ before and after being pruned. We repeat the process until the supernetwork $\mathcal{N}_k$ becomes a single-path network. We summarize our algorithm in \cref{algo:metaNTK-NAS}.

To evaluate the architecture searched by MetaNTK-NAS, we train the network on the same dataset where the search is conducted. For the training, we keep the same number of cells and the initial number of channels used in the search stage. As for the training recipe, we use RFS \cite{tian2020rethink} (without their additional knowledge distillation trick), an FSL method in the fashion of pre-training+finetuning, since it is an efficient first-order optimization method. RFS greatly reduces the training time, and GPU memory compared to higher-order optimization methods such as MAML \cite{maml}. To compare our searched cells with MetaNAS, we evaluate cells found by MetaNAS\footnote{We evaluate the cells used in their large-scale experiments.} with the same RFS training-evaluate pipeline and present as MetaNTK{\small (retrained)} in Table \ref{tab:benchmark}.
\vspace{-1.3em}
\paragraph{Hyper-parameters}During the search phase, for computing NTK, MetaNTK, and the number of linear regions, we adopt data augmentation used in TE-NAS \cite{TE-NAS}. In addition to the MAML-induced MetaNTK (MAML-kernel) we derive in Theorem \ref{thm:MetaNTK}, we also implement another MetaNTK induced by ANIL (ANIL-kernel)\cite{raghu2019rapid}, a simplified version of MAML. We use MAML-kernel and ANIL-kernel for the 5-cells and 8-cells experiments, respectively. More details about hyperparameters such as batch size, dropout rate \cite{srivastava2014dropout} and normalization \cite{batchnorm,groupnorm} can be found in Appendix \ref{supp:exp}.

\begin{table*}[t!]
    \begin{center}
    \vspace{-2em}
    \resizebox{1\linewidth}{!}
    {
    \begin{tabular}{@{}llc@{}c@{}c@{}c@{}cccc@{}c@{}cccc@{}}
    \\
    \hline
    \toprule
    &&& & \multicolumn{4}{c}{\textbf{mini-ImageNet 5-way}} & &\multicolumn{4}{c}{\textbf{tiered-ImageNet 5-way}} \\
    \cmidrule{5-8} \cmidrule{10-13}
    \textbf{Model} &\textbf{Arch.} &\textbf{~\#Cells}&\textbf{Train} &\textbf{\#Param.}&\textbf{~Search Cost} & \textbf{1-shot} & \textbf{5-shot} & &\textbf{\#Param.}&\textbf{~Search Cost}&  \textbf{1-shot} & \textbf{5-shot}  \\
    \midrule
    MAML\cite{maml}& Conv4 & -& MAML& 30k & -&48.70$\pm$1.84 &63.11$\pm$0.92 & &30k&-& 51.67$\pm$1.81 & 70.30$\pm$1.75 \\
    ANIL \cite{raghu2019rapid}& Conv4 & -& ANIL & 30k&-& 48.0$\pm$0.7 & 62.2$\pm$0.5 & &-&-& - & - \\
    MetaOptNet\cite{metaOptNet} & ResNet-12 & - & {\small MetaOptNet} & 12.5M & - & 62.64$\pm$0.61 & 78.63$\pm$0.46 && 12.7M & -& 65.99$\pm$0.72 & 81.56$\pm$0.53 \\
    
    RFS\cite{tian2020rethink} & ResNet-12 & - & RFS & 12.5M & - & 62.02$\pm$0.63 & 79.64$\pm$0.44 && 12.7M & -& 69.74$\pm$0.72 & 84.41$\pm$0.55 \\
    AutoMeta\cite{auto-meta} & Cells & - & Reptile  & 100k & 2688 hr& 57.6$\pm$0.2 & 74.7$\pm$0.2 && - & - & - &-\\
    T-NAS++\cite{TNAS} & Cells & 2 & FOMAML  & 27k & 48 hr& 54.11$\pm$1.35 & 69.59$\pm$0.85 && - & - & - &-\\
    MetaNAS\cite{metaNAS} & Cells & 5 & Reptile &  1.1M & 168 hr& 63.1$\pm$0.3 & 79.5$\pm$0.2 && - & - & - &-\\
    \midrule
    \midrule
    MetaNAS{\small (retrained)}\textcolor{red}{$^\dagger$} & Cells & 5 & RFS & 2.01M &  168 hr & \textbf{64.24$\pm$0.11} & 79.75$\pm$0.13 & & 2.17M & 168 hr & 70.16$\pm$0.09 & 84.99$\pm$0.22 \\
    MetaNTK-NAS & Cells & 5 & RFS & 1.77M & 1.54 hr & 63.88$\pm$0.81 & \textbf{80.07$\pm$0.45} & &2.22M &2.20 hr & \textbf{71.12$\pm$0.49} & \textbf{85.71$\pm$0.22} \\
    \midrule
    MetaNAS{\small (retrained)}\textcolor{red}{$^\dagger$} & Cells & 8 & RFS & 3.53M & 168 hr & 63.88$\pm$0.23 & 79.88$\pm$0.14 &  & 3.70M & 168 hr &\textbf{72.32$\pm$0.02}&\textbf{86.48$\pm$0.06} \\
    MetaNTK-NAS & Cells & 8 & RFS & 3.21M & 1.92 hr & \textbf{64.26$\pm$0.14} & \textbf{80.35$\pm$0.12} &  & 4.78M & 2.73 hr &\textbf{72.37$\pm$0.79} &\textbf{86.43$\pm$0.53} \\
    \bottomrule
    \hline
    \end{tabular}
    }
    \caption{
    \textbf{Comparison on few-shot image classification benchmarks.} Average few-shot test classification accuracy (\%) $\pm$ standard deviation. The first 4 rows are few-shot learning algorithms on standard networks, and the following 3 rows (AutoMeta, T-NAS++, MetaNAS) are prior NAS methods designed for few-shot learning. The last 4 rows are from our experiments: we stack cells presented by MetaNAS \cite{metaNAS} and searched by our MetaNTK-NAS in the same way (stacks of 5 cells or 8 cells), and train both of them with RFS \cite{tian2020rethink}. \\
    \textcolor{red}{$\dagger$} (\textit{i}) The authors of MetaNAS \cite{metaNAS} only present one cell structure that they manually select over multiple structures in the search process on miniImageNet, and train it for 3 runs. Thus this structure can be seen as the \textit{best} structure they obtain. In contrast, we run the search-train-evaluate pipeline of MetaNTK-NAS for 3 independent runs and take average of the test accuracy. (\textit{ii}) \cite{metaNAS} does not run on tieredImageNet, thus we stack the cells they searched on miniImageNet and train the structure on tieredImageNet using RFS. We believe MetaNAS cells work relatively well on tiered-ImageNet because mini- and tiered-ImageNet are both subsets of ImageNet, sharing lots of similarities. \\}
    \label{tab:benchmark}
    \vspace{-32pt}
    \end{center}
\end{table*}
\vspace{-1.3em}
\paragraph{Optimization Setup} Following \cite{tian2020rethink}, we adopt SGD optimizer with a momentum of 0.9 and a weight decay of $0.0005$. We train all models for 100 epochs and 60 epochs on miniImageNet and tieredImagenet, respectively. For miniImageNet, the learning rate is 0.02 initially, and it is decayed by 10x at epoch 60 and 80. For tieredImageNet, the learning rate is 0.01 initially, and it is decayed by 10x at epochs 30, 40, 50. 
\vspace{-1.3em}
\paragraph{Model Selection.} We always take the model checkpoints at the end of training for evaluation.
\vspace{-1.3em}
\paragraph{Evaluation on Test Tasks} Following \cite{tian2020rethink}, for any test task, we remain hidden layers intact and finetune the linear output layer of each network on labeled support samples, and then evaluate its prediction accuracy on the query samples as the test accuracy. Similar to \cite{tian2020rethink}, we mostly use $\ell_2$ regularized cross-entropy loss to finetune the linear layer, while we also use $\ell_2$ regularized hinge loss in some 1-shot cases. More details regarding the setup and hyperparameters of the evaluation stage can be found in Appendix \ref{supp:exp}.
\vspace{-2em}
\paragraph{Code} Our code is written in PyTorch \cite{pytorch}. For the search stage of NAS, we build our code upon the released codebase of \cite{TE-NAS}. Opacus \cite{opacus} is used to compute per-sample-gradients efficiently to further construct MetaNTK. For the training and evaluation stages, we adopt the code of \cite{tian2020rethink}.
\vspace{-1em}
\paragraph{Hardware} Most of our experiments were run on NVIDIA V100s, and the rest were run on NVIDIA RTX 3090s. Each experiment is run on a single GPU at a time. The search cost of MetaNTK-NAS is benchmarked on V100s.

\begin{figure*}[t]
    \vspace{-2em}
    \small
        \centering
        \begin{subfigure}[b]{.5\textwidth}  
            \centering
            \includegraphics[width=.95\linewidth]{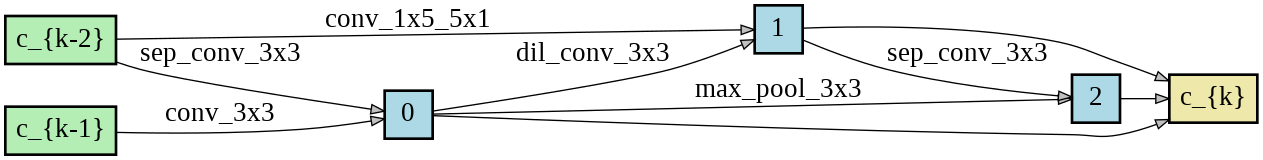}
            \caption%
            {\small Normal cell (5 cells, miniImageNet).} 
            \label{fig:s1}
        \end{subfigure}
        \begin{subfigure}[b]{.480\textwidth}  
            \centering 
            \includegraphics[width=.95\linewidth]{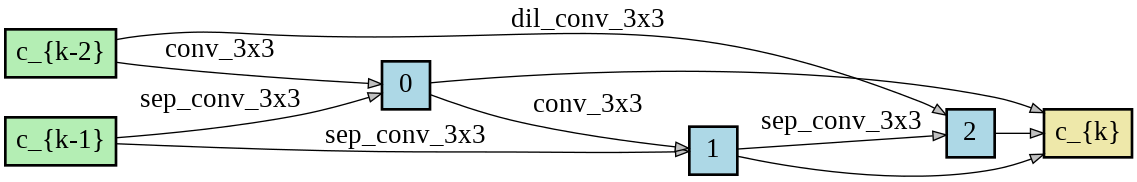} 
            \caption[]%
            {{\small Reduction cell (5 cells, miniImageNet).}}    
            \label{fig:s2}
        \end{subfigure}
        
                \begin{subfigure}[b]{.5\textwidth}  
            \centering
            \includegraphics[width=.95\linewidth]{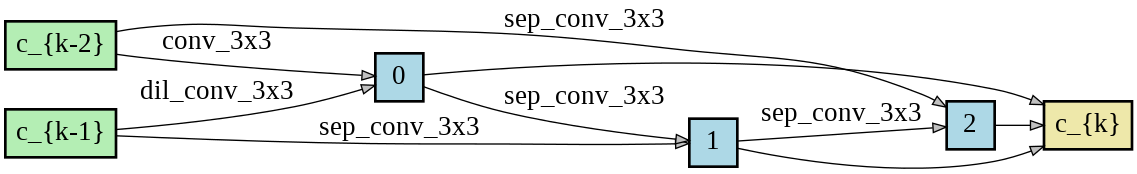}
            \caption[]%
            {{\small Normal cell (5 cells, tieredImageNet).}}    
            \label{fig:s1}
        \end{subfigure}
        \begin{subfigure}[b]{.480\textwidth}  
            \centering 
            \includegraphics[width=.95\linewidth]{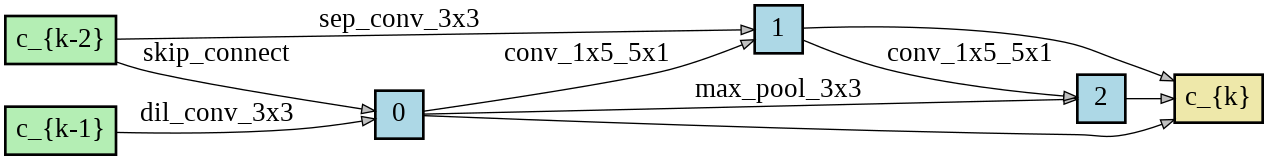} 
            \caption[]%
            {{\small Reduction cell (5 cells, tieredImageNet).}}    
            \label{fig:s2}
        \end{subfigure}
        
        \caption{Examples of normal and reduction cells found by MetaNTK-NAS that are used for the evaluation. Green, blue and yellow boxes denote outputs of the previous cells, intermediate nodes and output of the current cell, respectively. See \cite{DARTS} for more details.} 
        \label{fig: normal_cells}
\end{figure*}

\subsection{Experiment Results}
\begin{table*}[ht!]
\centering
\small
\resizebox{0.85\linewidth}{!}{

\setlength{\tabcolsep}{5.5pt}
\begin{tabular}{c|ccc|cc|cc}

\hline
\toprule

& & & \textbf{Linear} & \multicolumn{2}{c}{\textbf{1-shot}} & \multicolumn{2}{c}{\textbf{5-shot}}\\
\textbf{Model} & \textbf{MetaNTK} & \textbf{NTK} & \textbf{Region} & 5 Cells & 8 Cells & 5 Cells & 8 Cells  \\

\midrule
MetaNTK-NAS & \checkmark & &\checkmark                    
 & \textbf{63.88 $\pm$ 0.81} & \textbf{64.26 $\pm$ 0.14} & \textbf{80.07 $\pm$ 0.45} & \textbf{80.35 $\pm$ 0.12} \\
TE-NAS\cite{TE-NAS}& & \checkmark &  \checkmark
& 62.51 $\pm$ 0.42 & 63.51 $\pm$ 0.16 & 79.02 $\pm$ 0.35 & 79.86 $\pm$ 0.42 \\
Linear Region & &  & \checkmark
& 62.96 $\pm$ 1.05 & 63.99 $\pm$ 0.76 & 79.13 $\pm$ 0.96 & 79.88 $\pm$ 0.51 \\
Random Cell & &  &
& 62.55 $\pm$ 1.15 & 62.76 $\pm$ 0.83  & 78.90 $\pm$ 0.58 & 79.18 $\pm$ 0.72 \\
\bottomrule

\end{tabular}

}

\caption{\textbf{Ablation study on mini-ImageNet.} The row ``Linear Region'' stands for NAS with only the number of linear regions. The row ``Random Cell'' evaluates stacks of randomly sampled candidate cells from the search space.}
\label{tab:ablation}
\vspace{-1.5em}
\end{table*}
\paragraph{Empirical Results} Each experiment of MetaNTK-NAS and MetaNAS(retrained) is repeated over at least 3 runs of different random seeds. We evaluate each trained model on 1000 test tasks randomly sampled from the test set, and report the mean and standard deviation of the test accuracy.
\vspace{-2em}
\paragraph{Performance Comparison} In Table \ref{tab:benchmark}, we compare MetaNTK-NAS with multiple meta-learning algorithms on miniImageNet and tieredImageNet under 5-way few-shot classification setting. Among these algorithms, MetaOptNet is the state-of-the-art gradient-based meta-learning algorithm, and RFS is an efficient few-shot learning algorithm, while AutoMeta, T-NAS++, and MetaNAS are existing NAS algorithms specifically designed for few-shot learning. Compared to MetaOptNet and RFS, which use ResNet-12 as backbone architecture, our method outperforms them with a much fewer number of parameters (2.5x$\sim$6x fewer). Besides, our MetaNTK-NAS achieves comparable or better performance than MetaNAS, the state-of-the-art NAS algorithm for few-shot learning across various settings on the two datasets (cf. the last 4 rows of \cref{tab:benchmark}). Examples of searched cells are also shown in \cref{fig: normal_cells}.
\vspace{-1.1em}

\textbf{Efficiency Comparison} While achieving competitive performance, our method enjoys \textit{100x} faster search speed (1.54 GPU hours search cost compared to 168 GPU hours search cost of MetaNAS on miniImageNet for 5-cells structures). Since MetaNAS cells are obtained by training supernetwork in the small-scale regime of 100k parameters on miniImageNet, we believe that comparing our search cost of 5-cells setup on miniImageNet to theirs is fair. We also believe MetaNTK-NAS is likely to have an even larger improvement in search speed compared to MetaNAS on tieredImageNet\footnote{The experiments of the original MetaNAS paper \cite{metaNAS} is limited to miniImageNet and Omniglot (a much smaller dataset).} or larger datasets. Therefore, we can conclude that \textit{MetaNTK-NAS is comparable to or better than the state-of-the-art of NAS methods for few-shot learning with 100x faster search speed.}

\vspace{-0.1em}
\subsection{Ablation Experiments}
In this section, we conduct ablation studies to analyze the effects of different ingredients used in MetaNTK-NAS. We compare the following setups with different combinations of components: (a) we conduct a search with TE-NAS \cite{TE-NAS}, which searches with NTK and number of linear regions\footnote{We use the same number of training samples to compute MetaNTK and NTK for a fair comparison.}; (b) we conduct a search using only the number of linear regions (no NTK or meta-NTK is used); (c) we conduct a random search, where we randomly sample candidate cells from the search space.

\cref{tab:ablation} shows the results of our ablation study on miniImageNet. The improvement of MetaNTK-NAS over the rest methods indicates the usefulness of the MetaNTK condition number in NAS for few-shot learning. 
Notice that the NTK condition number seems to have adverse effects on the performance. It indicates that NTK, which is derived in supervised learning, does not fit few-shot learning well.

\vspace{-.4em}
\section*{Conclusion}
In this paper, we first focus on the \textit{optimization} properties of \textit{model-agnostic meta-learning} (MAML) equipped with deep neural networks (DNNs), and prove the \textit{global convergence} of MAML with over-parameterized deep neural nets. Based on the convergence analysis, we prove that in the infinite width limit of DNNs, MAML converges to a kernel regression with a new class of kernels, which we name as \textit{Meta Neural Tangent Kernel} (MetaNTK). 
Inspired by recent works that apply Neural Tangent Kernel (NTK) to NAS for supervised learning \cite{TE-NAS,KNAS}, we propose MetaNTK-NAS, a new NAS method for few-shot learning based on our derived MetaNTK. Empirically, we compare MetaNTK-NAS with prior works, and observe that the performance of MetaNTK-NAS is comparable or better than the state-of-the-art NAS methods for few-shot learning on miniImageNet and tieredImageNet, while enjoying 100x less search cost.

\vspace{-.5em}
\section*{Acknowledgement}
\vspace{-.5em}
This work is partially supported by NSF grant No.1910100 and NSF CNS 20-46726 CAR. This work utilizes resources supported by NSF Major Research Instrumentation program, grant No.1725729\cite{HAL}.

\newpage
\bibliographystyle{ieee_fullname}
\bibliography{egbib}
\newpage
\onecolumn
\appendix
\section*{Overview of the Appendix}
The appendix is divided into the following sections, 
\begin{itemize}
    \item Sec. \ref{supp:NTK-setup}: Describes the neural network setup and parameterization.
    \item Sec. \ref{supp:global-convergence}: Presents theoretical results and proof regarding the \textbf{global convergence} of gradient-based meta-learning.
    \item Sec. \ref{supp:GBML-output}: Derives the expression of MAML output.
    \item Sec. \ref{supp:MetaNTK}: Derives the equivalence between MAML and kernel regressions.
    \item Sec. \ref{supp:exp}: Presents more details of \textbf{experiments} in Sec. \ref{sec:exp}.
\end{itemize}

\section{Neural Network Setup}
\label{supp:NTK-setup}

In this paper, we consider a fully-connected feed-forward network with $L$ hidden
layers. Each hidden layer has width $l_{i}$, for $i = 1, ..., L$. The readout layer (i.e. output layer) has width $l_{L+1} = k$. At each layer $i$, for arbitrary input $x\in \mathbb R^{d}$, we denote the pre-activation and post-activation functions by $h^i(x), z^i(x)\in\mathbb R^{l_i}$. The relations between layers in this network are
\begin{align}
\label{eq:recurrence}
\begin{cases}
    h^{i+1}&=z^i W^{i+1} + b^{i+1}
    \\
    z^{i+1}&=\activation \left(h^{i+1}\right) 
    \end{cases}
    \,\, \textrm{and} 
    \,\,
    \begin{cases}
  W^i_{\mu,\nu}& =  \omega_{\mu \nu}^i \sim \mathcal{N} (0, \frac {\sigma_\omega} {\sqrt{l_i}} )
    \\
    b_\nu^i &= \beta_\nu^i \sim \mathcal{N} (0,\sigma_b  )
\end{cases}
,
\end{align}
where $W^{i+1}\in \mathbb R^{l_i\times l_{i+1}}$ and $b^{i+1}\in\mathbb R^{l_{i+1}}$ are the weight and bias of the layer, $\omega_{\mu \nu}^l$ and $ b_\nu^l $ are trainable variables drawn i.i.d. from zero-mean Gaussian distributions at initialization (i.e., $\frac{\sws}{l_i}$ and $\sbs$ are variances for weight and bias, and $\activation$ is a point-wise activation function.

\section{Proof of Global Convergence for Gradient-Based Meta-Learning with Deep Neural Networks}
In this section, we will prove the global convergence for gradient-based meta-learning with over-parameterized neural nets. To prove the global convergence theorem, we introduce several key lemmas first, i.e., Lemma \ref{lemma:local-Liphschitzness}, \ref{lemma:bounded_init_loss}, \ref{lemma:kernel-convergence}. Specifically, the subsections of this section are formulated as follows.
\begin{itemize}
    \item Sec. \ref{supp:global-convergence:lemma-proof}: Present several helper lemmas with proof.
    \item Sec. \ref{supp:global-convergence:proof-liphscitzness}: Provides the proof of Lemma \ref{lemma:local-Liphschitzness}.
    \item Sec. \ref{supp:lemma-proof:bounded_init_loss}: Provides the proof of Lemma \ref{lemma:bounded_init_loss}.
    \item Sec. \ref{supp:global-convergence:kernel-convergence}: Provides the proof of Lemma \ref{lemma:kernel-convergence}.
    \item Sec. \ref{supp:global-convergence:theorem-proof}: Proves the global convergence theorem for MAML, i.e., Theorem \ref{thm:global-convergence:restated} (restated version of Theorem \ref{thm:global-convergence}).
\end{itemize}
\label{supp:global-convergence}

Notice that in this section, we consider the standard parameterization scheme of neural networks shown in (\ref{eq:recurrence}).

The global convergence theorem, Theorem \ref{thm:global-convergence}, depends on several assumptions and lemmas. The assumptions are listed below. After that, we present the lemmas and the global convergence theorem, with proofs in Appendix \ref{supp:global-convergence:lemma-proof},\ref{supp:lemma-proof:bounded_init_loss},\ref{supp:global-convergence:lemma-proof:kernel-convergence} and \ref{supp:global-convergence:theorem-proof}. For Corollary \ref{corr:GBML-output}, we append its proof to Appendix \ref{supp:GBML-output}.

\begin{assumption}[Bounded Input Norm]\label{assum:input-norm<=1}
$\forall \sX\in \X$, for any sample $\sx \in \sX$, $\|\sx\|_2 \leq 1$. Similarly, $\forall \sX' \in \X'$, for any sample $\sx' \in \sX'$, $\|\sx'\|_2\leq 1$. (This is equivalent to a input normalization operation, which is common in data preprocessing.)
\end{assumption}

\begin{assumption}[Non-Degeneracy]\label{assum:non-degeneracy}
The meta-training set $(\X,\Y)$ and the meta-test set $(\X',\Y')$ are both contained in some compact set. Also, $\X$ and $\X'$ are both non-degenerate, i.e. $\forall \sX, \widetilde{\sX} \in \X$, $\sX\neq \widetilde{\sX}$, and $\forall \sX', \widetilde{\sX}' \in \X'$, $\sX'\neq \widetilde{\sX}'$.
\end{assumption}

\begin{assumption}[Same Width for Hidden Layers]\label{assum:same-width} 
All hidden layers share the same width, $l$, i.e., $l_1=l_2=\dots=l_L = l$. 

\end{assumption}

\begin{assumption}[Full-Rank]\label{assum:full-rank} The kernel $\metaNTK$ defined in Lemma \ref{lemma:kernel-convergence} is full-rank.
\end{assumption}

These assumptions are common, and one can find similar counterparts of them in the literature for supervised learning \cite{lee2019wide,CNTK}. In particular, notice that Assumption \ref{assum:same-width} is just for simplicity purpose without loss generality. In fact, one can directly set $l = \min_{i\in[L]} l_i$ as the minimum width across hidden layers, and all theoretical results in this paper still hold true \cite{lee2019wide}.

As defined in the main text, $\theta$ is used to represent the neural net parameters. For convenience, we define some short-hand notations:
\begin{align}
    f_t(\cdot) &= f_{\theta_t}(\cdot)\\
    F_t(\cdot) &= F_{\theta_t}(\cdot)\\
    f(\theta) &= f_{\theta}(\X) = ((f_{\theta}(\sX_i))_{i=1}^{N} \\
    F(\theta) &= F_{\theta}(\XXY) =((F_{\theta}(\xxyi))_{i=1}^{N}\\
    g(\theta) &=  F_{\theta}(\XXY) - \Y \\
    J(\theta) &= \nabla_{\theta} F(\theta) = \nabla_\theta F_\theta(\XXY)
\end{align}
and
\begin{align}
    \mc L (\theta_t)&=\ell(F(\theta_t),\Y)=\frac 1 2 \|g(\theta_t)\|_2^2 \label{eq:supp:train_loss}\\
    \metantk_t &= \frac{1}{l} \nabla F_{\theta_t}(\XXY) \nabla F_{\theta_t}(\XXY) = \frac{1}{l} J(\theta) J(\theta)^\top \label{eq:supp:emp_MNK}
\end{align}
where we use the $\ell_2$ loss function $\ell (\hat{y},y)=\frac 1 2 \|\hat y - y\|^2_2$ in the definition of training loss $\loss(\theta_t)$ in (\ref{eq:supp:train_loss}), and the $\metantk_t$ in (\ref{eq:supp:emp_MNK}) is based on the definition\footnote{There is a typo in the definition of $\metantk_{\theta}(\cdot,\star)$ in Sec. \ref{sec:global-convergence}: a missing factor $\frac{1}{l}$. The correct definition should be $\metantk_{\theta}(\cdot,\star)= \frac{1}{l}\nabla_\theta F_\theta(\cdot)\nabla_\theta F_\theta (\star)^\top$. Similarly, the definition of $\metaNTK$ in Theorem \ref{thm:global-convergence} also missis this factor: the correct version is $\metaNTK = \frac{1}{l}\lim_{l \rightarrow \infty} J(\theta_0) J(\theta_0)^\top$} of $\metantk_\theta(\cdot,\star)$ in Sec. \ref{sec:global-convergence}.

Below, Lemma \ref{lemma:local-Liphschitzness} proves the Jacobian $J$ is locally Lipschitz, Lemma \ref{lemma:bounded_init_loss} proves the training loss at initialization is bounded, and Lemma \ref{lemma:kernel-convergence} proves $\metantk_0$ converges in probability to a deterministic kernel matrix with bounded positive eigenvalues. Finally, with these lemmas, we can prove the global convergence of MAML in Theorem \ref{thm:global-convergence:restated}.

\begin{lemma}[{\bf Local Lipschitzness of Jacobian}]
\label{lemma:local-Liphschitzness} 
For arbitraily small $\delta > 0$, then there exists $K>0$ and $l^*>0$ such that: $\forall ~C>0$ and $l > l^*$, the following inequalities hold true with probability at least $1-\delta$ over random initialization,
\begin{align}\label{eq:jacobian-lip}
 \forall \theta, \, \Bar \theta \in B(\theta_0, C l^{-\frac 1 2}),  \begin{cases}  
    \frac 1 {\sqrt l}\|J(\theta) - J(\Bar \theta)\|_{F} &\leq K\|\theta - \Bar \theta\|_2
    \\
    \\
    \frac 1 {\sqrt l} \|J(\theta)\|_{F} & \leq K 
    \end{cases}
\end{align}
where $B$ is a neighborhood defined as
\begin{align}
\label{def:B:supp}
    B(\theta_0, R) := \{\theta: \|\theta-\theta_0\|_2 < R\}.   
\end{align}
\end{lemma}
 \begin{proof}
 See Appendix \ref{supp:global-convergence:lemma-proof}.
 \end{proof}

 \begin{lemma}[\textbf{Bounded Initial Loss}] \label{lemma:bounded_init_loss}
For arbitrarily small $\delta_0 > 0$, there are constants $R_0>0$ and $l^*>0$ such that as long as the width $l > l^*$, with probability at least $(1-\delta_0)$ over random initialization,
\begin{align}
    \|g(\theta_0)\|_2 &= \|F_{\theta_0}(\XXY) - \Y\|_2\leq R_0,
\end{align}
which is also equivalent to
\begin{align*}
    \loss(\theta_0) = \frac 1 2 \|g(\theta_0)\|^2_2 \leq \frac 1 2 R_0^2.
\end{align*}
\end{lemma}
\begin{proof}
See Appendix \ref{supp:lemma-proof:bounded_init_loss}.
\end{proof}
\begin{lemma}[\textbf{Kernel Convergence}]\label{lemma:kernel-convergence}
Suppose the learning rates $\eta$ and $\lambda$ suffiently small. As the network width $l$ approaches infinity, $\metantk_0 = J(\theta_0) J(\theta_0)^\top$ converges in probability to a deterministic kernel matrix $\metaNTK$ (i.e., $\metaNTK=\lim_{l \rightarrow \infty} \metantk_0$), which is independent of $\theta_0$ and can be analytically calculated. Furthermore, the eigenvalues of $\metaNTK$ is bounded as, $0 < \lev (\metaNTK) \leq \Lev (\metaNTK) < \infty$.
\end{lemma}
\begin{proof}
See Appendix \ref{supp:global-convergence:lemma-proof:kernel-convergence}.
\end{proof}

Note the update rule of gradient descent on $\theta_t$ with learning rate $\eta$ can be expressed as
\begin{align}\label{eq:gd&jacobian}
    \theta_{t+1} = \theta_t - \eta J(\theta_t)^{\top}g(\theta_t).
\end{align}
The following theorem proves the global convergence of MAML under the update rule of gradient descent.
\begin{theorem}[\textbf{Global Convergence} \textit{(Theorem \ref{thm:global-convergence} restated)}]\label{thm:global-convergence:restated}
Denote $\lev=\lev(\metaNTK)$ and $\Lev=\Lev(\metaNTK)$. For any $\delta_0 > 0$ and $\eta_0 < \frac{2}{\Lev + \lev}$,
there exist $\lss>0$, $\Lambda \in\mathbb N$, $K>1$, and $\lambda_0>0$, such that: for width $l\geq \Lambda$, running gradient descent with learning rates $\eta = \frac {\eta_0}{l}$ and $\lambda < \frac{\lambda_0}{l}$ over random initialization, the following inequalities hold true with probability at least $(1 - \delta_0)$:
\begin{align}
     \sum_{j=1}^{t}\|\theta_j - \theta_{j-1}\|_2  &\leq  \frac {3K \lss}{\lev}  l^{-\frac 1 2} \label{eq:convergence-parameters:supp}\\
    \sup_{t} \|  \metantk_0 - \metantk_t\|_F  &\leq \frac {6K^3\lss}{\mins}  l^{-\frac 1 2}
    \label{eq:convergence-metantk:supp}
\end{align}
and
\begin{align}
g(\theta_t) = \|F(\theta_t) - \Y \|_2 \leq  \left(1 - \frac {\eta_0 \mins}{3}\right)^t \lss\,,
\end{align}
which leads to
\begin{align}
\loss (\theta_t) = \frac{1}{2}\|F(\theta_t) - \Y \|_2^2 \leq  \left(1 - \frac {\eta_0 \mins}{3}\right)^{2t} \frac{\lss^2}{2}\,,
    \label{eq:convergence-loss:supp}
\end{align}
indicating the training loss converges to zero at a linear rate.
\end{theorem}
\begin{proof}
See Appendix \ref{supp:global-convergence:theorem-proof}.
\end{proof}

In the results of Theorem \ref{thm:global-convergence:restated} above, (\ref{eq:convergence-parameters:supp}) considers the optimization trajectory of network parameters, and show the parameters move locally during training. (\ref{eq:convergence-metantk:supp}) indicates the kernel matrix $\metantk_t$ changes slowly. Finally, (\ref{eq:convergence-loss:supp}) demonstrates that the training loss of MAML decays exponentially to zero as the training time evolves, indicating convergence to global optima at a linear rate.

\subsection{Helper Lemmas}
\label{supp:global-convergence:lemma-proof}

\begin{lemma}[]\label{lemma:helper-zero-F-norm}
As the width $l\rightarrow \infty$, for any vector $\mathbf{a} \in \bR^{m\times 1}$ that $\|\mathbf{a}\|_F \leq C$ with some constant $C>0$, we have 
\begin{align}\label{eq:NTK-grad-vec-Frob}
\|\nabla_\theta \ntk_\theta(x,X') \cdot \mathbf{a}\|_F \rightarrow 0
\end{align}
where $\theta$ is randomly intialized parameters.
\end{lemma}
\begin{proof}
Notice that
\begin{align}
\ntk_\theta(x,X') = \frac{1}{l}\underbrace{\nabla_\theta f_\theta(\overbrace{x}^{\in\bR^{ d}})}_{\in \bR^{ 1\times D}}\cdot \underbrace{\nabla_\theta f_\theta(\overbrace{X'}^{\in\bR^{m\times d}})^\top}_{\in \bR^{D \times m}} \in \bR^{ 1\times m}
\end{align}
with gradient as
\begin{align}\label{eq:helper-zero-F-norm:NTK-grad}
\nabla_\theta \ntk_\theta(x,X') &= \frac{1}{l}\underbrace{\nabla^2_\theta f_\theta(x)}_{\in \bR^{ 1 \times D \times D}} \cdot\underbrace{\nabla_\theta f_\theta(X')^\top}_{\in \bR^{D\times m}} + \frac{1}{l}\underbrace{ \nabla_\theta f_\theta(x)}_{\in \bR^{1\times D}} \cdot \underbrace{\nabla^2_\theta f_\theta(X')^\top}_{\in \bR^{D\times m \times D}} \in \bR^{1 \times m \times D}
\end{align}
where we apply a \textit{dot product} in the \textit{first two dimensions} of 3-tensors and matrices to obtain matrices.

Then, it is obvious that our goal is to bound the Frobenius Norm of
\begin{align}\label{eq:helper-zero-F-norm:NTK-grad:tensor-contraction}
    \nabla_\theta \ntk_\theta(x,X') \cdot \mathbf{a} = \left(\frac{1}{l}\nabla^2_\theta f_\theta(x) \cdot\nabla_\theta f_\theta(X')^\top\right) \cdot \mathbf{a} + \left(\frac{1}{l}\nabla_\theta f_\theta(x) \cdot\nabla^2_\theta f_\theta(X')^\top\right) \cdot \mathbf{a}
\end{align}

Below, we prove that as the width $l\rightarrow \infty$, the first and second terms of \eqref{eq:helper-zero-F-norm:NTK-grad:tensor-contraction} both have vanishing Frobenius norms, which finally leads to the proof of \eqref{eq:NTK-grad-vec-Frob}.
\begin{itemize}
    \item \textit{First Term of (\ref{eq:helper-zero-F-norm:NTK-grad:tensor-contraction})}. Obviously, reshaping $\nabla^2_\theta f_\theta(x) \in \bR^{1\times D\times D}$ as a $\bR^{D \times D}$ matrix does not change the Frobenius norm $\|\frac{1}{l}\nabla^2_\theta f_\theta(x) \cdot\underbrace{\nabla_\theta f_\theta(X')^\top}_{\in \bR^{D\times m}}\|_F$ (in other words, $\|\frac{1}{l}\underbrace{\nabla^2_\theta f_\theta(x)}_{\in \bR^{ {1 \times D \times D}}} \cdot\underbrace{\nabla_\theta f_\theta(X')^\top}_{\in \bR^{D\times m}}\|_F = \|\frac{1}{l}\underbrace{\nabla^2_\theta f_\theta(x)}_{\in \bR^{D \times D}} \cdot\underbrace{\nabla_\theta f_\theta(X')^\top}_{\in \bR^{D\times m}}\|_F$). 
    
    By combining the following three facts,
    \begin{enumerate}
        \item $\|\frac{1}{\sqrt{l}}\underbrace{\nabla^2_\theta f_\theta(x)}_{\in \bR^{D \times D}} \|_{op} \rightarrow 0 $ indicated by \cite{hessian-ntk},
        \item the matrix algebraic fact $\|HB\|_F \leq \|H\|_{op} \|B\|_F$,
        \item the bound $\|\frac{1}{\sqrt{l}} \nabla_\theta f_\theta(\cdot)\|_F < constant$ from \cite{lee2019wide},
    \end{enumerate}
    one can easily show that the first term of (\ref{eq:helper-zero-F-norm:NTK-grad}) has vanishing Frobenius norm, i.e., 
    \begin{align}\label{eq:NTK-grad:first-term-Frob}
    \|\frac{1}{l}\nabla^2_\theta f_\theta(x) \cdot\nabla_\theta f_\theta(X')^\top\|_F \rightarrow 0
    \end{align}
    Then, obviously,
    \begin{align}\label{eq:NTK-grad:first-term-vec-Frob}
        \|\left(\frac{1}{l}\nabla^2_\theta f_\theta(x) \cdot\nabla_\theta f_\theta(X')^\top\right) \cdot \mathbf{a}\|_F \leq \|\frac{1}{l}\nabla^2_\theta f_\theta(x) \cdot\nabla_\theta f_\theta(X')^\top\|_F \|\mathbf{a}\|_F \rightarrow 0
    \end{align}
    \item \textit{Second Term of (\ref{eq:helper-zero-F-norm:NTK-grad:tensor-contraction}).} From \cite{hessian-ntk}, we know that
    \begin{align}\label{eq:NTK-grad:Hessian-vec-op-norm}
    \|\underbrace{\frac{1}{\sqrt{l}} \nabla^2_\theta f_\theta(X')^\top}_{\in \bR^{D \times m \times D}}\cdot \underbrace{\mathbf{a}}_{\in \bR^{m \times 1}} \|_{op} \rightarrow 0
    \end{align}
    Then, similar to the derivation of (\ref{eq:NTK-grad:first-term-Frob}), we have
    \begin{align}\label{eq:NTK-grad:second-term-vec-Frob}
        \|\left(\frac{1}{l}\nabla_\theta f_\theta(x) \cdot\nabla^2_\theta f_\theta(X')^\top\right) \cdot \mathbf{a}\|_F \leq \overbrace{\|\frac{1}{\sqrt{l}}\nabla_\theta f_\theta(x)\|_F}^{\leq constant} \cdot\overbrace{\|\nabla^2_\theta f_\theta(X')^\top\cdot\mathbf{a}\|_{op}}^{\rightarrow 0} \rightarrow 0
    \end{align}
    
    \item Finally, combining (\ref{eq:NTK-grad:first-term-vec-Frob}) and (\ref{eq:NTK-grad:second-term-vec-Frob}), we obtain (\ref{eq:NTK-grad-vec-Frob}) by
    \begin{align}
        \|\nabla_\theta \ntk_\theta(x,X') \cdot \mathbf{a}\|_F&\leq \|\left(\frac{1}{l}\nabla^2_\theta f_\theta(x) \cdot\nabla_\theta f_\theta(X')^\top\right) \cdot \mathbf{a}\|_F \nonumber\\ 
        &~+\|\left(\frac{1}{l}\nabla_\theta f_\theta(x) \cdot\nabla^2_\theta f_\theta(X')^\top\right) \cdot \mathbf{a}\|_F\nonumber\\
        &\rightarrow 0
    \end{align}
\end{itemize}
\end{proof}

\begin{lemma}
\label{lemma:helper-local-Liphschitzness}
Given any task $\task = (\xyxy)$ and randomly initialized parameters $\theta$, as the width $l \rightarrow \infty$, for any $x \in X$, where $x\in \bR^{d}$ and $X\in \bR^{n\times d}$, we have
\begin{align}\label{eq:lemma:helper-local-Liphschitzness}
    \|\nabla_\theta \left(\ntk_{\theta}(x,X')\ntk_{\theta}^{-1} (I - e^{-\lambda \ntk_{\theta} \tau})  \right) (f_\theta(\sX') - Y')\|_F \rightarrow 0~,
\end{align}
and furthermore,
\begin{align}\label{eq:lemma:helper-local-Liphschitzness:stack-of-x}
     \|\nabla_\theta \left(\ntk_{\theta}(X,X')\ntk_{\theta}^{-1} (I - e^{-\lambda \ntk_{\theta} \tau})  \right)  (f_\theta(\sX') - Y')\|_F \rightarrow 0~.
\end{align}
\end{lemma}

\begin{proof}[Proof of Lemma \ref{lemma:helper-local-Liphschitzness}]

~\\
\textbf{Overview.} In this proof, we consider the expression
\begin{align}
    &\quad \nabla_\theta \left(\ntk_{\theta}(x,X')\ntk_{\theta}^{-1} (I - e^{-\lambda \ntk_{\theta} \tau})  \right)  (f_\theta(\sX') - Y') \label{eq:lemma:helper-local-Liphschitzness:grad-all}\\
    &= ~~~\nabla_\theta \left(\ntk_{\theta}(x,X')\right)\ntk_{\theta}^{-1} (I - e^{-\lambda \ntk_{\theta} \tau})  (f_\theta(\sX') - Y') \label{eq:lemma:helper-local-Liphschitzness:grad-1}\\
    &\quad +\ntk_{\theta}(x,X')\left(\nabla_\theta \ntk_{\theta}^{-1}\right) (I - e^{-\lambda \ntk_{\theta} \tau})  (f_\theta(\sX') - Y')\label{eq:lemma:helper-local-Liphschitzness:grad-2}\\
    &\quad +\ntk_{\theta}(x,X') \ntk_{\theta}^{-1} \left(\nabla_\theta (I - e^{-\lambda \ntk_{\theta} \tau}) \right)  (f_\theta(\sX') - Y')\label{eq:lemma:helper-local-Liphschitzness:grad-3},
\end{align}
and we prove the terms of (\ref{eq:lemma:helper-local-Liphschitzness:grad-1}), (\ref{eq:lemma:helper-local-Liphschitzness:grad-2}) and (\ref{eq:lemma:helper-local-Liphschitzness:grad-3}) all have vanishing Frobenius norm. Thus, (\ref{eq:lemma:helper-local-Liphschitzness:grad-all}) also has vanishing Frobenius norm in the infinite width limit, which is exactly the statement of (\ref{eq:lemma:helper-local-Liphschitzness}). This indicates that \eqref{eq:lemma:helper-local-Liphschitzness:stack-of-x} also has a vanishing Frobenius norm, since $\ntk_\theta(X,X')$ can be seen as a stack of $n$ copies of $\ntk_\theta(x,X')$, where $n$ is a finite constant.

\paragraph{Step I.}Each factor of \eqref{eq:lemma:helper-local-Liphschitzness:grad-all} has bounded Frobenius norm.
\begin{itemize}
    \item $\|\ntk_{\theta}(x,X')\|_F$. It has been shown that $\|\frac{1}{\sqrt l} \nabla_\theta f(\cdot)\|_F \leq constant$ in \cite{lee2019wide}, thus we have $\|\ntk_{\theta}(x,X')\|_F = \| \frac{1}{l} \nabla_\theta f(x) \nabla_\theta f(X')^\top\|_F\leq \|\frac{1}{\sqrt l} \nabla_\theta f(x)\|_F\|\frac{1}{\sqrt l} \nabla_\theta f(X')\|_F\leq constant$.
    \item $\|\ntk_{\theta}^{-1}\|_F$. It has been shown that $\ntk_\theta$ is positive definite with positive least eigenvalue \cite{ntk,CNTK}, thus $\|\ntk_\theta^{-1}\|_F \leq constant$.
    \item $\|I - e^{-\lambda \ntk_{\theta} \tau}\|_F$. \cite{cao2019generalization} shows that largest eigenvalues of $\ntk_\theta$ are of $O(L)$, and we know $\ntk_\theta$ is positive definite \cite{ntk,CNTK}, thus it is obvious the eigenvalues of $I - e^{-\lambda \ntk_{\theta} \tau}$ fall in the set $\{z~|~0<z<1\}$. Therefore, certainly we have $\|I - e^{-\lambda \ntk_{\theta} \tau}\|_F \leq constant$.
    \item $\| f_\theta(\sX') - Y'\|_F$. \cite{lee2019wide} shows that $\| f_\theta(\sX') - Y'\|_2 \leq constant$, which indicates that $\| f_\theta(\sX') - Y'\|_F \leq constant$.
\end{itemize}
In conclusion, we have shown
\begin{align}\label{eq:lemma:helper-local-Liphschitzness:factors-bounded-frob}
\|\ntk_{\theta}(x,X')\|_F,\|\ntk_{\theta}^{-1}\|_F,\|I - e^{-\lambda \ntk_{\theta} \tau}\|_F, \|f_\theta(\sX')-Y'\|_F \leq constant
\end{align}

\paragraph{Step II.} Bound (\ref{eq:lemma:helper-local-Liphschitzness:grad-1}).

 Without loss of generality, let us consider the neural net output dimension $k=1$ in this proof, i.e., $f_\theta: \mathbb{R}^d \mapsto \mathbb{R}$. (Note: with $k>1$, the only difference is that $\nabla_{\theta} f(X') \in \bR^{mk\times D}$, which has no impact on the proof). Then, we have
\begin{align}
\ntk_\theta(x,X') = \frac{1}{l}\underbrace{\nabla_\theta f_\theta(\overbrace{x}^{\in\bR^{ d}})}_{\in \bR^{ 1\times D}}\cdot \underbrace{\nabla_\theta f_\theta(\overbrace{X'}^{\in\bR^{m\times d}})^\top}_{\in \bR^{D \times m}} \in \bR^{ 1\times m}
\end{align}
with gradient as
\begin{align}\label{eq:NTK-grad}
\nabla_\theta \ntk_\theta(x,X') &= \frac{1}{l}\underbrace{\nabla^2_\theta f_\theta(x)}_{\in \bR^{ 1 \times D \times D}} \cdot\underbrace{\nabla_\theta f_\theta(X')^\top}_{\in \bR^{D\times m}} + \frac{1}{l}\underbrace{ \nabla_\theta f_\theta(x)}_{\in \bR^{1\times D}} \cdot \underbrace{\nabla^2_\theta f_\theta(X')^\top}_{\in \bR^{D\times m \times D}} \in \bR^{1 \times m \times D}
\end{align}
where we apply a \textit{dot product} in the \textit{first two dimensions} of 3-tensors and matrices to obtain matrices.

Based on (\ref{eq:lemma:helper-local-Liphschitzness:factors-bounded-frob}), we know that $$\|\overbrace{\ntk_{\theta}^{-1} (I - e^{-\lambda \ntk_{\theta} \tau})  (f_\theta(\sX') - Y')}^{\in \bR^{m \times 1}}\|_F\leq \|\ntk_{\theta}^{-1} \|_F \|I - e^{-\lambda \ntk_{\theta} \tau}\|_F  \|f_\theta(\sX') - Y'\|_F \leq constant~.$$ Then, applying (\ref{eq:NTK-grad-vec-Frob}), we have
\begin{align}\label{eq:lemma:helper-local-Liphschitzness:term-1}
\|\nabla_\theta \left(\ntk_{\theta}(x,X')\right)\cdot \ntk_{\theta}^{-1} (I - e^{-\lambda \ntk_{\theta} \tau})  (f_\theta(\sX') - Y') \|_F \rightarrow 0
\end{align}

\paragraph{Step III.} Bound (\ref{eq:lemma:helper-local-Liphschitzness:grad-2}) and (\ref{eq:lemma:helper-local-Liphschitzness:grad-3})

\begin{itemize}
    \item Bound (\ref{eq:lemma:helper-local-Liphschitzness:grad-2}): $\ntk_{\theta}(x,X')\left(\nabla_\theta \ntk_{\theta}^{-1}\right) (I - e^{-\lambda \ntk_{\theta} \tau}) (f_\theta(\sX') - Y')$.
    
    Clearly, $\underbrace{\nabla_\theta \ntk_\theta^{-1}}_{m \times m \times D} = -\underbrace{\ntk_\theta^{-1}}_{\in \bR^{m \times m}} \cdot \underbrace{(\nabla_\theta \ntk_\theta)}_{\in \bR^{m \times m \times D}} \cdot \underbrace{\ntk_\theta^{-1}}_{\bR^{m \times m}}$, where we apply a dot product in the first two dimensions of the 3-tensor and matrices. 
    
    Note that $\nabla_\theta \ntk_\theta = \sqrt{1}{l} \nabla_\theta^2 f_\theta (X') \cdot \nabla_\theta f_\theta (X')^\top+ \sqrt{1}{l}\nabla_\theta f_\theta (X') \cdot \nabla_\theta^2 f_\theta (X')^\top$. Obviously, by (\ref{eq:NTK-grad:Hessian-vec-op-norm}) and (\ref{eq:lemma:helper-local-Liphschitzness:factors-bounded-frob}), we can easily prove that 
    \begin{align}\label{eq:lemma:helper-local-Liphschitzness:term-2}
        \|\ntk_{\theta}(x,X')\left(\nabla_\theta \ntk_{\theta}^{-1}\right) (I - e^{-\lambda \ntk_{\theta} \tau}) (f_\theta(\sX') - Y')\|_F \rightarrow 0
    \end{align}
    \item Bound (\ref{eq:lemma:helper-local-Liphschitzness:grad-3}): $\ntk_{\theta}(x,X') \ntk_{\theta}^{-1} \left(\nabla_\theta (I - e^{-\lambda \ntk_{\theta} \tau}) \right) (f_\theta(\sX') - Y')$
    
    Since $\underbrace{\nabla_\theta (I - e^{-\lambda \ntk_\theta \tau })}_{\in \bR^{m \times m \times D}} = \lambda \tau \cdot \underbrace{e^{-\lambda \ntk_\theta \tau}}_{\in \bR^{m \times m}}\cdot \underbrace{\nabla_\theta \ntk_\theta}_{\in \bR^{m \times m \times D}}$, we can easily obtain the following result by (\ref{eq:NTK-grad:Hessian-vec-op-norm}) and (\ref{eq:lemma:helper-local-Liphschitzness:factors-bounded-frob}),
    \begin{align}\label{eq:lemma:helper-local-Liphschitzness:term-3}
         \|\ntk_{\theta}(x,X') \ntk_{\theta}^{-1} \left(\nabla_\theta (I - e^{-\lambda \ntk_{\theta} \tau}) \right) (f_\theta(\sX') - Y')\|_F \rightarrow 0
    \end{align}

\end{itemize}
\paragraph{Step IV.}Final result: prove (\ref{eq:lemma:helper-local-Liphschitzness:grad-all}) and \eqref{eq:lemma:helper-local-Liphschitzness:stack-of-x}.

Combining (\ref{eq:lemma:helper-local-Liphschitzness:term-1}), (\ref{eq:lemma:helper-local-Liphschitzness:term-2}) and (\ref{eq:lemma:helper-local-Liphschitzness:term-3}), we can prove \eqref{eq:lemma:helper-local-Liphschitzness:grad-all}
\begin{align}
&\quad \|\nabla_\theta \left(\ntk_{\theta}(x,X')\ntk_{\theta}^{-1} (I - e^{-\lambda \ntk_{\theta} \tau})  \right) (f_\theta(\sX') - Y')|_F \\
&\leq \|\nabla_\theta \left(\ntk_{\theta}(x,X')\right)\ntk_{\theta}^{-1} (I - e^{-\lambda \ntk_{\theta} \tau}) (f_\theta(\sX') - Y') \|_F \nonumber\\
&+  \|\ntk_{\theta}(x,X')\left(\nabla_\theta \ntk_{\theta}^{-1}\right) (I - e^{-\lambda \ntk_{\theta} \tau})(f_\theta(\sX') - Y')\|_F \nonumber\\
&+\|\ntk_{\theta}(x,X') \ntk_{\theta}^{-1} \left(\nabla_\theta (I - e^{-\lambda \ntk_{\theta} \tau}) \right) (f_\theta(\sX') - Y')\|_F\nonumber \\
&\rightarrow 0 
\end{align}

Then, since $\ntk_\theta(X,X')$ can be seen as a stack of $n$ copies of $\ntk_\theta(x,X')$, where $n$ is a finite constant, we can easily prove \eqref{eq:lemma:helper-local-Liphschitzness:stack-of-x} by
\begin{align}
    &\quad \|\nabla_\theta \left(\ntk_{\theta}(X,X')\ntk_{\theta}^{-1} (I - e^{-\lambda \ntk_{\theta} \tau})  \right) (f_\theta(\sX') - Y')\|_F\\ 
    &\leq \sum_{i\in [n]} \|\nabla_\theta \left(\ntk_{\theta}(x_i,X')\ntk_{\theta}^{-1} (I - e^{-\lambda \ntk_{\theta} \tau})  \right) (f_\theta(\sX') - Y')\|_F \nonumber\\
    &\rightarrow 0
\end{align}
where we denote $X = (x_i)_{i=1}^n$.
\end{proof}
\subsection{Proof of Lemma \ref{lemma:local-Liphschitzness}}\label{supp:global-convergence:proof-liphscitzness}
\begin{proof}[Proof of Lemma \ref{lemma:local-Liphschitzness}]
Consider an arbitrary task $\task = (\xyxy)$. Given sufficiently large width $l$, for any parameters in the neighborhood of the initialization, i.e., $\theta \in \B(\theta_0, C l^{-1/2})$, based on \cite{lee2019wide}, we know the meta-output can be decomposed into a terms of $f_\theta$,
\begin{align}\label{eq:lemma:jacobian:0}
    F_\theta(\xxy) = f_\theta(X) - \ntk_{\theta}(X,X')\ntk_{\theta}^{-1} (I - e^{-\lambda \ntk_{\theta} \tau})(f_\theta(X')-Y'),
\end{align}
where $\ntk_\theta(X,X') = \frac{1}{l}\nabla_\theta f_\theta(X) \nabla_\theta f_\theta(X')^\top$, and $\ntk_{\theta}\equiv \ntk_\theta(X',X')$ for convenience.

Then, we consider $\nabla_\theta F_\theta(\xxy)$, the gradient of $F_\theta(\xxy)$ in (\ref{eq:lemma:jacobian:0}),
\begin{align}\label{eq:lemma:jacobian:full-F-grad}
    \nabla_{\theta} F_\theta(\xxy)
    &= \nabla_{\theta}f_\theta(\sX) - \ntk_{\theta}(X,X')\ntk_{\theta}^{-1} (I - e^{-\lambda \ntk_{\theta} \tau}) \nabla_{\theta}f_\theta(\sX') \nonumber\\
    &\quad - \nabla_\theta \left(\ntk_{\theta}(X,X')\ntk_{\theta}^{-1} (I - e^{-\lambda \ntk_{\theta} \tau})  \right) (f_\theta(\sX') - Y')
\end{align}

By Lemma \ref{lemma:helper-local-Liphschitzness}, we know the last term of \eqref{eq:lemma:jacobian:full-F-grad} has a vanishing Frobenius norm as the width increases to infinity. Thus, for any $\varepsilon > 0$ and $0<\delta<1$, there exists $l^* >0 $ s.t. for width $l> l^*$, with probability at least $1-\delta$, the last term of \eqref{eq:lemma:jacobian:full-F-grad} is of $\cO(\varepsilon)$, i.e.,
\begin{align}\label{eq:lemma:jacobian:F-grad}
    \nabla_{\theta} F_\theta(\xxy)
    = \nabla_{\theta}f_\theta(\sX) - \ntk_{\theta}(X,X')\ntk_{\theta}^{-1} (I - e^{-\lambda \ntk_{\theta} \tau}) \nabla_{\theta}f_\theta(\sX') + \cO(\varepsilon)
\end{align}
Since $\cO(\varepsilon)$ is of a negligible order, we do not carry it in the remaining proof steps for simplicity, and it does not affect the correctness of the derivations (since the bounds of this Lemma are probabilistic).

Now, let us consider the SVD decomposition on $\frac{1}{\sqrt l}\nabla_\theta f_\theta(X') \in \mathbb R^{km \times D}$, where $X' \in \mathbb{R}^{k \times m}$ and $\theta \in \mathbb{R}^D$. such that $\frac{1}{\sqrt l}\nabla_\theta f_\theta(X') = U \Sigma V^\top$, where $U\in \mathbb R^{km \times km},V\in \mathbb R^{D \times km}$ are orthonormal matrices while $\Sigma \in \mathbb R^{km \times km}$ is a diagonal matrix. Note that we take $km \leq D$ here since the width is sufficiently wide.

Then, since $\ntk_\theta = \frac{1}{l}\nabla_\theta f_\theta(X') \nabla_\theta f_\theta(X')^\top= U \Sigma V^\top V \Sigma U^\top = U \Sigma^2 U^\top$, we have $\ntk^{-1}_\theta = U \Sigma^{-2} U^\top$. Also, by Taylor expansion, we have 
\begin{align}\label{eq:lemma:jacobian:taylor}
I - e^{-\lambda\ntk_\theta \tau} = I - \sum_{i=0}^\infty \frac{(-\lambda \tau)^i}{i!} (\ntk_\theta)^i = U\left( I - \sum_{i=0}^\infty \frac{(-\lambda \tau)^i}{i!} (\Sigma)^i \right)U^\top = U\left(I - e^{-\lambda \Sigma \tau}\right) U^\top.
\end{align}

With these results of SVD, (\ref{eq:lemma:jacobian:F-grad}) becomes
\begin{align}
    &\quad \nabla_{\theta} F((\xxy),\theta) \nonumber\\
    &= \nabla_{\theta}f_\theta(\sX) - \frac{1}{l}\nabla_{\theta} f_\theta(X) \nabla_{\theta} f_\theta(X')^\top \ntk_{\theta}^{-1} (I - e^{-\lambda \ntk_{\theta} \tau}) \nabla_{\theta}f_\theta(\sX')\nonumber\\
    &=\nabla_{\theta}f_\theta(\sX) - \frac{1}{l}\nabla_{\theta} f_\theta(X) (\sqrt{l}V \Sigma U^\top) (U \Sigma^{-2} U^\top) [ U\left(I - e^{-\lambda \Sigma \tau}\right) U^\top] (\sqrt{l}U \Sigma V^\top) \nonumber\\
    &=\nabla_{\theta}f_\theta(\sX) - \nabla_\theta f_\theta(X) V \Sigma^{-1} \left(I - e^{-\lambda \Sigma \tau}\right) \Sigma V^\top\nonumber\\
    &=\nabla_{\theta}f_\theta(\sX) - \nabla_\theta f_\theta(X) V \left(I - e^{-\lambda \Sigma \tau}\right)  V^\top \nonumber\\
    &=\nabla_{\theta}f_\theta(\sX) - \nabla_\theta f_\theta(X) (I - e^{-\lambda \conjntk_{\theta} \tau})\nonumber \\
    &=\nabla_\theta f_\theta(X) e^{-\lambda \conjntk_{\theta} \tau} \label{eq:lemma:jacobian:F-grad:-2}
\end{align}
where $\conjntk_\theta\equiv \conjntk_\theta(X',X')=\frac{1}{l} \nabla_\theta f_\theta(X')^\top \nabla_\theta f_\theta(X') \in \mathbb R^{D \times D}$, and the step (\ref{eq:lemma:jacobian:F-grad:-2}) can be easily obtained by a Taylor expansion similar to (\ref{eq:lemma:jacobian:taylor}).

Note that $\conjntk_\theta$ is a product of $\nabla_\theta f_\theta(X')^\top$ and its transpose, hence it is positive semi-definite, and so does $e^{-\lambda \conjntk \tau}$. By eigen-decomposition on $\conjntk$, we can easily see that the eigenvalues of $e^{-\lambda \conjntk \tau}$ are all in the range $[0,1)$ for arbitrary $\tau>0$. Therefore, it is easy to get that for arbitrary $\tau > 0$,
\begin{align}\label{eq:lemma:jacobian:F-grad:norm:0}
    \|\nabla_{\theta} F((\xxy),\theta)\|_F = \|\nabla_\theta f_\theta(X) e^{-\lambda \conjntk_{\theta} \tau}\|_F \leq \|\nabla_\theta f_\theta(X)\|_F 
\end{align}
By Lemma 1 of \cite{lee2019wide}, we know that there exists a $K_0>0$ such that for any $X$ and $\theta$,
\begin{align}\label{eq:lemma:jacobian:1}
    \|\frac{1}{\sqrt l} \nabla f_\theta(X) \|_F \leq K_0.
\end{align}
Combining (\ref{eq:lemma:jacobian:F-grad:norm:0}) and (\ref{eq:lemma:jacobian:1}), we have
\begin{align}
    \|\frac{1}{\sqrt l}\nabla_{\theta} F((\xxy),\theta)\|_F \leq \|\frac{1}{\sqrt l}\nabla_{\theta} f_\theta(X)\|_F \leq K_0,
\end{align}
which is equivalent to 
\begin{align}
    \frac{1}{\sqrt l}\|J(\theta)\|_F \leq K_0
\end{align}

Now, let us study the other term of interest, $\|J(\theta)-J(\Bar \theta)\|_F=\|\frac{1}{\sqrt l}\nabla_{\theta} F((\xxy),\theta)-\frac{1}{\sqrt l}\nabla_{\theta} F((\xxy),\Bar \theta)\|_F$, where $\theta,\Bar \theta \in \B(\theta_0, C l^{-1/2})$. 

To bound $\|J(\theta)-J(\Bar \theta)\|_F$, let us consider
\begin{align}
    &\quad ~\|\nabla_{\theta} F((\xxy),\theta)-\nabla_{\theta} F((\xxy),\Bar \theta)\|_{op} \label{eq:lemma:jacobian:grad_diff:expression}\\
    &= \|\nabla_\theta f_\theta(X) e^{-\lambda \conjntk_{\theta} \tau} - \nabla_{\Bar \theta} f_{\Bar \theta}(X) e^{-\lambda \conjntk_{\Bar \theta} \tau}\|_{op} \nonumber\\
    &=\frac{1}{2}\|\left(\nabla_\theta f_\theta(X) - \nabla_{\Bar \theta} f_{\Bar \theta}(X)\right)\left( e^{-\lambda \conjntk_{\theta} \tau} + e^{-\lambda \conjntk_{\Bar \theta} \tau}\right)\\
    &\quad +\left(\nabla_\theta f_\theta(X) + \nabla_{\Bar \theta} f_{\Bar \theta}(X)\right)\left( e^{-\lambda \conjntk_{\theta} \tau} - e^{-\lambda \conjntk_{\Bar \theta} \tau}\right)\|_{op} \nonumber\\
    &\leq \frac{1}{2} \|\nabla_\theta f_\theta(X) - \nabla_{\Bar \theta} f_{\Bar \theta}(X)\|_{op}\| e^{-\lambda \conjntk_{\theta} \tau} + e^{-\lambda \conjntk_{\Bar \theta} \tau}\|_{op}\\
    &\quad +\frac{1}{2}\|\nabla_\theta f_\theta(X) + \nabla_{\Bar \theta} f_{\Bar \theta}(X)\|_{op}\| e^{-\lambda \conjntk_{\theta} \tau} - e^{-\lambda \conjntk_{\Bar \theta} \tau}\|_{op}\nonumber \\
    &\leq \frac{1}{2} \|\nabla_\theta f_\theta(X) - \nabla_{\Bar \theta} f_{\Bar \theta}(X)\|_{op} \left(\| e^{-\lambda \conjntk_{\theta} \tau}\|_{op} + \|e^{-\lambda \conjntk_{\Bar \theta} \tau}\|_{op}\right) \label{eq:lemma:jacobian:grad_diff:0}\\
    &\quad +\frac{1}{2}\left(\|\nabla_\theta f_\theta(X)\|_{op} + \|\nabla_{\Bar \theta} f_{\Bar \theta}(X)\|_{op}\right)\| e^{-\lambda \conjntk_{\theta} \tau} - e^{-\lambda \conjntk_{\Bar \theta} \tau}\|_{op}  \label{eq:lemma:jacobian:grad_diff:1}
\end{align}
It is obvious that $\| e^{-\lambda \conjntk_{\theta} \tau}\|_{op},\|e^{-\lambda \conjntk_{\Bar \theta} \tau}\|_{op} \leq 1$. Also, by the relation between the operator norm and the Frobenius norm, we have
\begin{align}\label{eq:lemma:jacobian:1.1}
    \|\nabla_\theta f_\theta(X) - \nabla_{\Bar \theta} f_{\Bar \theta}(X)\|_{op} \leq \|\nabla_\theta f_\theta(X) - \nabla_{\Bar \theta} f_{\Bar \theta}(X)\|_{F} 
\end{align}
Besides, Lemma 1 of \cite{lee2019wide} indicates that there exists a $K_1 > 0$ such that for any $X$ and $\theta,\Bar \theta \in \B(\theta_0, C l^{-1/2})$,
\begin{align}\label{eq:lee:jocobian:1}
    \|\frac{1}{\sqrt l} \nabla_\theta f_\theta(X) - \frac{1}{\sqrt l}\nabla_\theta f_{\Bar \theta}(X)\|_F \leq K_1 \|\theta - \Bar \theta\|_2
\end{align}
Therefore, (\ref{eq:lemma:jacobian:1.1}) gives
\begin{align}
    \|\nabla_\theta f_\theta(X) - \nabla_{\Bar \theta} f_{\Bar \theta}(X)\|_{op}  \leq K_1 \sqrt{l}\|\theta - \Bar \theta\|_2
\end{align}
and then (\ref{eq:lemma:jacobian:grad_diff:0}) is bounded as
\begin{align}\label{eq:lemma:jacobian:grad_diff:0:bound}
     \frac{1}{2} \|\nabla_\theta f_\theta(X) - \nabla_{\Bar \theta} f_{\Bar \theta}(X)\|_{op} \left(\| e^{-\lambda \conjntk_{\theta} \tau}\|_{op} + \|e^{-\lambda \conjntk_{\Bar \theta} \tau}\|_{op}\right) \leq K_1 \sqrt{l} \|\theta - \Bar \theta\|_2.
\end{align}
As for (\ref{eq:lemma:jacobian:grad_diff:1}), notice that $\|\cdot\|_{op} \leq \|\cdot\|_F$ and (\ref{eq:lemma:jacobian:1}) give us
\begin{align}\label{eq:lemma:jacobian:grad_diff:1:1}
    \frac{1}{2}\left(\|\nabla_\theta f_\theta(X)\|_{op} + \|\nabla_{\Bar \theta} f_{\Bar \theta}(X)\|_{op}\right) \leq \sqrt{l} K_0.
\end{align}

Then, to bound $\| e^{-\lambda \conjntk_{\theta} \tau} - e^{-\lambda \conjntk_{\Bar \theta} \tau}\|_{op}$ in (\ref{eq:lemma:jacobian:grad_diff:1}), let us bound the following first
\begin{align}
    \|\conjntk_\theta -\conjntk_{\Bar \theta}\|_{F} &= \|\frac{1}{l} \nabla_\theta f_\theta(X')^\top \nabla_\theta f_\theta(X') - \frac{1}{l} \nabla_{\Bar \theta} f_{\Bar \theta}(X')^\top \nabla_{\Bar \theta} f_{\Bar \theta}(X')\|_{F} \nonumber\\
    &= \frac{1}{l} \|\frac 1 2( \nabla_\theta f_\theta(X')^\top + \nabla_{\Bar \theta} f_{\Bar \theta}(X')^\top)(\nabla_\theta f_\theta(X') -\nabla_{\Bar \theta} f_{\Bar \theta}(X')) \nonumber\\
    &+ \frac 1 2( \nabla_\theta f_\theta(X')^\top -  \nabla_{\Bar \theta} f_{\Bar \theta}(X')^\top)(\nabla_\theta f_\theta(X')+\nabla_{\Bar \theta} f_{\Bar \theta}(X'))\|_F\nonumber\\
    &\leq \frac 1 l \|\nabla_\theta f_\theta(X') + \nabla_{\Bar \theta} f_{\Bar \theta}(X')\|_F\|\nabla_\theta f_\theta(X') -  \nabla_{\Bar \theta} f_{\Bar \theta}(X')\|_F\nonumber\\
    &\leq \frac 1 l \left(\|\nabla_\theta f_\theta(X')\|_F +\| \nabla_{\Bar \theta} f_{\Bar \theta}(X')_F\|\right) \|\nabla_\theta f_\theta(X') -  \nabla_{\Bar \theta} f_{\Bar \theta}(X')\|_F \nonumber\\
    &\leq 2 K_0 K_1 \|\theta - \Bar \theta\|_2
\end{align}
Then, with the results above and a perturbation bound\footnote{This bound is also derived in \cite{van1977sensitivity}.} on matrix exponentials from \cite{1977bounds}, we have
\begin{align}
    \| e^{-\lambda \conjntk_{\theta} \tau} - e^{-\lambda \conjntk_{\Bar \theta} \tau}\|_{op} 
    &\leq \|\conjntk_\theta -\conjntk_{\Bar \theta}\|_{op} \cdot \left(\lambda \tau e^{-\lambda \tau \cdot ( \|\conjntk_\theta\|_{op} - \|\conjntk_\theta -\conjntk_{\Bar \theta}\|_{op}}) \right)\nonumber\\
    &\leq \frac{\|\conjntk_\theta -\conjntk_{\Bar \theta}\|_{op}}{\|\conjntk_\theta\|_{op}- \|\conjntk_\theta -\conjntk_{\Bar \theta}\|_{op}} \nonumber\\
    &\leq \cO(\|\conjntk_\theta -\conjntk_{\Bar \theta}\|_{op})\nonumber \\
    &\leq 2 K_0 K_1 K_2 \|\theta - \Bar \theta\|_2 \label{eq:lemma:jacobian:grad_diff:1:2}
\end{align}
where we used the facts $\|\conjntk_\theta\|_{op} = \|\ntk_\theta\|_{op}\geq \cO(1)$ \cite{xiao2020dis,cao2019generalization} and $\|\conjntk_\theta -\conjntk_{\Bar \theta}\|_{op} \leq \cO(\|\theta - \Bar \theta\|_2 ) \leq \cO(\frac{1}{\sqrt{l}})$.

Hence, by (\ref{eq:lemma:jacobian:grad_diff:1:1}) and (\ref{eq:lemma:jacobian:grad_diff:1:2}), we can bound (\ref{eq:lemma:jacobian:grad_diff:1}) as
\begin{align}\label{eq:lemma:jacobian:grad_diff:1:bound}
\frac{1}{2}\left(\|\nabla_\theta f_\theta(X)\|_{op} + \|\nabla_{\Bar \theta} f_{\Bar \theta}(X)\|_{op}\right)\| e^{-\lambda \conjntk_{\theta} \tau} - e^{-\lambda \conjntk_{\Bar \theta} \tau}\|_{op} \leq 
2\sqrt{l} K_0^2 K_1 K_2 \|\theta - \Bar \theta\|_2 
\end{align}

Finally, with (\ref{eq:lemma:jacobian:grad_diff:0:bound}) and (\ref{eq:lemma:jacobian:grad_diff:1:bound}), we can bound (\ref{eq:lemma:jacobian:grad_diff:expression}) as
\begin{align*}
    \|\nabla_{\theta} F((\xxy),\theta)-\nabla_{\theta} F((\xxy),\Bar \theta)\|_{op} \leq (K_1 + 2 K_0^2 K_1 K_2 ) \sqrt{l} \|\theta - \Bar \theta\|_2 
\end{align*}

Finally, combining these bounds on (\ref{eq:lemma:jacobian:grad_diff:0}) and (\ref{eq:lemma:jacobian:grad_diff:1}), we know that
\begin{align}
    \|J(\theta)-J(\Bar \theta)\|_F&=\|\frac{1}{\sqrt l}\nabla_{\theta} F((\xxy),\theta)-\frac{1}{\sqrt l}\nabla_{\theta} F((\xxy),\Bar \theta)\|_F\nonumber\\
    &\leq \frac{\sqrt{kn}}{\sqrt l} \|\nabla_{\theta} F((\xxy),\theta)-\nabla_{\theta} F((\xxy),\Bar \theta)\|_{op}\nonumber\\
    &\leq \sqrt{kn} (K_1+ 2K_0^2 K_1 K_2) \|\theta - \Bar \theta\|_2\label{eq:lemma:jacobian:grad_diff:final}
\end{align}

Define $K_3 = \sqrt{kn} (K_1+ 2K_0^2 K_1 K_2 ) $, we have 
\begin{align}
    \|J(\theta)-J(\Bar \theta)\|_F \leq K_3 \|\theta - \Bar \theta\|_2
\end{align}

Taking $K=\max \{K_0,K_3\}$ completes the proof.
\end{proof}

\subsection{Proof of Lemma \ref{lemma:bounded_init_loss}}\label{supp:lemma-proof:bounded_init_loss}
\begin{proof}[Proof of Lemma \ref{lemma:bounded_init_loss}]
It is known that $f_{\theta_0}(\cdot)$ converges in distribution to a mean zero Gaussian with the covariance $\mc K$ determined by the parameter initialization \cite{lee2019wide}. As a result, for arbitrary $\delta_1 \in (0,1)$ there exist constants $l_1 > 0$ and $R_1 > 0$, such that: $\forall~ l \geq l_1$, over random initialization, the following inequality holds true with probability at least $(1-\delta_1)$,
\begin{align}\label{eq:convergence:init-error-bounded}
    \|f_{\theta_0}(X)-Y\|_2, \|f_{\theta_0}(X')-Y'\|_2 \leq R_1
\end{align}

We know that $ \forall \task = (\xyxy) \in \D$,
$$F_{\theta_0}(\xxy) = f_{\theta_0'}(X)$$
where $\theta_0'$ is the parameters after $\tau$-step update on $\theta_0$ over the meta-test task $(\sX',\sY')$:
\begin{align}
    &\theta_\tau=\theta', ~~\theta_0=\theta, \nonumber\\
    &\theta_{i+1} = \theta_{i} - \lambda \nabla_{\theta_{i}} \ell(f_{\theta_{i}}(\sX'),\sY') ~~ \forall i=0,...,\tau-1,
\end{align}
Suppose the learning rate $\lambda$ is sufficiently small, then similar to \eqref{eq:lemma:jacobian:0}, we have
\begin{align}
F_{\theta_0}(\xxy) = f_{\theta_0}(X) + \ntk_0(X,X') \ntk_0^{-1} (I-e^{-\lambda \ntk_0 \tau})(f_{\theta_0}(X')-Y').
\end{align}
where $\ntk_0(\cdot,\star) = \nabla_{\theta_0} f_{\theta_0}(\cdot) \nabla_{\theta_0} f_{\theta_0}(\star)^\top$ and we use a shorthand $\ntk_0 \equiv \ntk_0(X',X') $.

\cite{ntk} proves that for sufficiently large width, $\ntk_0$ is positive definite and converges to $\NTK$, the Neural Tangent Kernel, a full-rank kernel matrix with bounded positive eigenvalues. Let $\lev(\NTK)$ and $\Lev(\NTK)$ denote the least and largest eigenvalue of $\NTK$, respectively. Then, it is obvious that for a sufficiently over-parameterized neural network, the operator norm of $\ntk(X,X') \ntk^{-1} (I-e^{-\lambda \ntk \tau})$ can be bounded based on $\lev(\NTK)$ and $\Lev(\NTK)$. Besides, \cite{CNTK,lee2019wide} demonstrate that the neural net output at initialization, $f_{\theta_0}(\cdot)$, is a zero-mean Gaussian with small-scale covaraince. Combining these results and (\ref{eq:convergence:init-error-bounded}), we know there exists $R(R_1,N,\lev(\NTK),\Lev(\NTK))$ such that 
\begin{align}
\|F_{\theta_0}(\xxy)-Y\|_2\leq R(R_1,N,\lev(\NTK),\Lev(\NTK))
\end{align}
By taking an supremum over $R(R_1,N,\lev,\Lev)$ for each training task in $\{\task_i=(\xyxyi)\}_{i\in[N]}$, we can get $R_2$ such that $\forall i \in [N]$
\begin{align}
\|F_{\theta_0}(\xxyi)-Y_i\|_2\leq R_2
\end{align}
and for $R_0 = \sqrt{N} R_2$, define $\delta_0$ as some appropriate scaling of $\delta_1$, then the following holds true with probability $(1-\delta_0)$ over random initialization,
\begin{align}
\|g(\theta_0)\|_2 &=\sqrt{\sum_{\xyxy \in \D}\|F((\xxy),\theta_0)-y\|_2^2} \leq R_0
\end{align}
\end{proof}
\subsection{Proof of Lemma \ref{lemma:kernel-convergence}}\label{supp:global-convergence:kernel-convergence}
\begin{proof}[Proof of Lemma \ref{lemma:kernel-convergence}]
The learning rate for meta-adaption, $\lambda$, is sufficiently small, so the inner-loop optimization becomes \textit{continuous-time} gradient descent. Based on \cite{lee2019wide}, for any task $\task=(\xyxy)$,
\begin{align}\label{eq:lemma-proof:kernel-convergence:F_0}
 F_0(\xxy) =f_0(\sX) + \ntk_0(\xx)\widetilde{\T}_{\ntk_0}^\lambda(\sX',\tau)\left(\sY' - f_0(\sX')\right),
\end{align}
where $\ntk_0(\cdot,\star)= \frac{1}{l}\nabla_{\theta_0} f_0(\cdot) \nabla_{\theta_0} f_0(\star)^\top$, and $\widetilde{T}^{\lambda}_{\ntk_0} (\cdot,\tau) \coloneqq \ntk_0(\cdot,\cdot)^{-1}(I-e^{-\lambda \ntk_0(\cdot,\cdot) \tau})$.

Then, we consider $\nabla_{\theta_0} F_{0}(\xxy)$, the gradient of $F_{0}(\xxy)$ in (\ref{eq:lemma-proof:kernel-convergence:F_0}). By Lemma \ref{lemma:helper-local-Liphschitzness}, we know that for sufficiently wide networks, the gradient of $F_{0}(\xxy)$ becomes 
\begin{align}\label{eq:MetaNTK:F_t-grad}
    \nabla_{\theta_0} F_0(\xxy)= \nabla_{\theta_0}f_0(\sX) - \ntk_0(\xx)\T_{\ntk_0}^\lambda(\sX',\tau) \nabla_{\theta_0}f_0(\sX')
\end{align}
Since $\metantk_0 \equiv \metantk_0((\XXY),(\XXY)) = \frac 1 l \nabla_{\theta_0} F_0(\XXY) \nabla_{\theta_0} F_0(\XXY)^\top$ and $F_0(\XXY) = (F_0(\xxyi))_{i=1}^N \in \mathbb{R}^{knN}$, we know $\metantk_0$ is a block matrix with $N\times N$ blocks of size $kn \times kn$. For $i,j \in [N]$, the $(i,j)$-th block can be denoted as $[\metantk_0]_{ij}$ such that
\begin{align}
    [\metantk_0]_{ij} &= \frac{1}{l} \nabla_{\theta_0} F_0(\xxyi) \nabla_{\theta_0} F_0(\xxyj)^\top \nonumber\\
    &= \quad \frac{1}{l}\nabla_{\theta_0}f_0(X_i) \nabla_{\theta_0}f_0(X_j)^\top \nonumber\\
    &\quad + \frac{1}{l} \ntk_0(\xxi)\widetilde{\T}_{\ntk_0}^\lambda(X_i',\tau) \nabla_{\theta_0}f_0(X_i')  \nabla_{\theta_0}f_0(X_j')^\top \widetilde{\T}_{\ntk_0}^\lambda(X_j',\tau)^\top \ntk_0(X_j',X_j)\nonumber\\
    &\quad - \frac{1}{l} \nabla_{\theta_0}f_0(X_i)  \nabla_{\theta_0}f_0(X_j')^\top \widetilde{\T}_{\ntk_0}^\lambda(X_j',\tau)^\top \ntk_0(X_j',X_j) \nonumber\\
    &\quad - \frac{1}{l} \ntk_0(\xxi)\widetilde{\T}_{\ntk_0}^\lambda(X_i',\tau) \nabla_{\theta_0}f_0(X_i')\nabla_{\theta_0}f_0(X_j)^\top \nonumber\\
    &= \quad \ntk_0 (X_i,X_j) \nonumber\\
    &\quad + \ntk_0(\xxi)\widetilde{\T}_{\ntk_0}^\lambda(X_i',\tau) \ntk_0(X_i',X_j')  \widetilde{\T}_{\ntk_0}^\lambda(X_j',\tau)^\top \ntk_0(X_j',X_j)\nonumber\\
    &\quad -  \ntk_0(X_i,X_j') \widetilde{\T}_{\ntk_0}^\lambda(X_j',\tau)^\top \ntk_0(X_j',X_j) \nonumber\\
    &\quad -\ntk_0(\xxi)\widetilde{\T}_{\ntk_0}^\lambda(X_i',\tau) \ntk_0(X_i',X_j) 
\end{align}
where we used the equivalences $\ntk_0(\cdot, \star) = \ntk_0(\star,\cdot)^\top$ and $\frac{1}{l} \nabla_{\theta_0} f_0(\cdot) \nabla_{\theta_0} f_0(\star) = \ntk_0(\cdot, \star)$.

By Algebraic Limit Theorem for Functional Limits, we have
\begin{align}
     &\quad \lim_{l\rightarrow\infty}[\metantk_0]_{ij} \nonumber\\
     &= \lim_{l\rightarrow\infty} \ntk_0 (X_i,X_j) \nonumber\\
    &\quad + \lim_{l\rightarrow\infty}\ntk_0(\xxi)\T_{\lim_{l\rightarrow\infty}\ntk_0}^\lambda(X_i',\tau) \lim_{l\rightarrow\infty} \ntk_0(X_i',X_j')  \T_{\lim_{l\rightarrow\infty} \ntk_0}^\lambda(X_j',\tau)^\top \lim_{l\rightarrow\infty} \ntk_0(X_j',X_j)\nonumber\\
    &\quad -  \lim_{l\rightarrow\infty} \ntk_0(X_i,X_j') \T_{\lim_{l\rightarrow\infty}\ntk_0}^\lambda(X_j',\tau)^\top \lim_{l\rightarrow\infty}\ntk_0(X_j',X_j) \nonumber\\
    &\quad -\lim_{l\rightarrow\infty}\ntk_0(\xxi)\T_{\lim_{l\rightarrow\infty}\ntk_0}^\lambda(X_i',\tau) \ntk_0(X_i',X_j) \nonumber\\
        &= \quad \NTK (X_i,X_j) \nonumber\\
    &\quad + \NTK(\xxi)\widetilde{\T}_{\NTK}^\lambda(X_i',\tau) \NTK(X_i',X_j')  \widetilde{\T}_{\NTK}^\lambda(X_j',\tau)^\top \NTK(X_j',X_j)\nonumber\\
    &\quad -  \NTK(X_i,X_j') \widetilde{\T}_{\NTK}^\lambda(X_j',\tau)^\top \NTK(X_j',X_j) \nonumber\\
    &\quad -\NTK(\xxi)\widetilde{\T}_{\NTK}^\lambda(X_i',\tau) \NTK(X_i',X_j) \label{eq:lemma-proof:kernel-convergence:metantk_ij}
\end{align}
where $\NTK(\cdot,\star) = \lim_{l\rightarrow\infty}\ntk_0 (\cdot,\star)$ is a deterministic kernel function, the Neural Tangent Kernel function (NTK) from the literature on supervised learning \cite{ntk,lee2019wide,CNTK}. Specifically, $\ntk_0 (\cdot,\star)$ converges to $\NTK(\cdot,\star)$ in probability as the width $l$ approaches infinity.

Hence, for any $i,j \in [N]$, as the width $l$ approaches infinity, $[\metantk_0]_{ij}$ converges in probability to a deterministic matrix $\lim_{l\rightarrow\infty} [\metantk_0]_{ij}$, as shown by (\ref{eq:lemma-proof:kernel-convergence:metantk_ij}). Thus, the whole block matrix $\metantk_0$ converges in probability to a deterministic matrix in the infinite width limit. Denote $\metaNTK =\lim_{l \rightarrow \infty} \metantk_0$, then we know $\metaNTK$ is a deterministic matrix with each block expressed as (\ref{eq:lemma-proof:kernel-convergence:metantk_ij}).

Since $\metantk_0 \equiv \metantk_0((\XXY),(\XXY)) = \frac 1 l \nabla_{\theta_0} F_0(\XXY) \nabla_{\theta_0} F_0(\XXY)^\top$, it is a symmetric square matrix. Hence all eigenvalues of $\metantk_0$ are greater or equal to $0$, which also holds true for $\metaNTK$. In addition, because of Assumption \ref{assum:full-rank}, $\metaNTK$ is positive definite, indicating $\lev(\metaNTK) > 0$. On the other hand, from \cite{CNTK}, we know diagonal entries and eigenvalues of $\NTK(\cdot,\star)$ are positive real numbers upper bounded by $\cO(L)$, as a direct result, it is easy to verify that the diagonal entries of the matrix $\metaNTK$ are also upper bounded, indicating $\Lev(\metaNTK) < \infty$. Hence, we have $0 < \lev(\metaNTK) < \Lev (\metaNTK) < \infty$.

\textbf{Extension.} It is easy to extend (\ref{eq:lemma-proof:kernel-convergence:metantk_ij}), the expression for $\metaNTK\equiv \lim_{l \rightarrow \infty}\metantk_0((\XXY),(\XXY)$, to more general cases. Specifically, we can express $\metaNTK(\cdot,\star)$ analytically for arbitrary inputs. To achieve this, let us define a kernel function,  $\SingleTaskmetaNTK: (\mathbb{R}^{n \times k} \times \mathbb{R}^{m\times k}) \times  (\mathbb{R}^{n \times k} \times \mathbb{R}^{m\times k}) \mapsto \mathbb{R}^{nk \times nk}$ such that
\begin{align}
    \SingleTaskmetaNTK((\cdot,\ast), (\bullet, \star)) &= \NTK(\cdot,\bullet) + \NTK(\cdot,\ast)\widetilde{\T}_{\NTK}^\lambda(\ast,\tau)\NTK(\ast,\star)\widetilde{\T}_{\NTK}^\lambda(\star,\tau)^\top \NTK(\star,\bullet) \nonumber\\
 &\quad -\NTK(\cdot,\ast)\widetilde{\T}_{\NTK}^\lambda(\ast,\tau) \NTK(\ast,\bullet) - \NTK(\cdot,\star) \widetilde{\T}_{\NTK}^\lambda(\star,\tau)^\top \NTK(\star,\bullet) .
\end{align}
Then, it is obvious that for $i,j\in[N]$, the $(i,j)$-th block of $\metaNTK$ can be expressed as $[\metaNTK]_{ij} = \SingleTaskmetaNTK((\xxi),(\xxj))$. 

For cases such as $\metaNTK((\xx),(\XX)) \in \mathbb{R}^{kn \times knN}$, it is also obvious that $\metaNTK((\xx),(\XX))$ is a block matrix that consists of $1 \times N$ blocks of size $k n \times k n$, with the $(1,j)$-th block as follows for $j \in [N]$, $$[\metaNTK((\xx),(\XX))]_{1,j}= \SingleTaskmetaNTK((X,X'),(X_j,X_j')).$$

\end{proof}
\label{supp:global-convergence:lemma-proof:kernel-convergence}
\subsection{Proof of Theorem \ref{thm:global-convergence:restated}}
\label{supp:global-convergence:theorem-proof}
\begin{proof}[Proof of Theorem \ref{thm:global-convergence:restated}]

Based on these lemmas presented above, we can prove Theorem \ref{thm:global-convergence:restated}.

Lemma \ref{lemma:bounded_init_loss} indicates that there exist $R_0$ and $l^*$ such that for any width $l \geq l^*$, the following holds true over random initialization with probability at least $(1-\delta_0/10)$,
\begin{align}\label{eq:thm:global-convergence:init-loss-bound}
    \|g(\theta_0)\|_2  \leq R_0 ~. 
\end{align}

Consider $C = \frac{3 K R_0}{\sigma}$ in Lemma \ref{lemma:local-Liphschitzness}. 

First, we start with proving (\ref{eq:convergence-parameters:supp}) and (\ref{eq:convergence-loss:supp}) by induction. Select $\wt l > l^*$ such that (\ref{eq:thm:global-convergence:init-loss-bound}) and (\ref{eq:jacobian-lip}) hold with probability at least
$1- \frac{\delta_0}{5}$ over random initialization for every $l \geq \wt l$. As $t=0$, by (\ref{eq:gd&jacobian}) and (\ref{eq:jacobian-lip}), we can easily verify that (\ref{eq:convergence-parameters:supp}) and (\ref{eq:convergence-loss:supp}) hold true
\begin{align*}
\begin{cases}
        \|\theta_1 - \theta_0 \|_2 &= \| -\eta J(\theta_0)^\top g(\theta_0) \|_2 \leq \eta \|J(\theta_0)\|_{op} \|g(\theta_0)\|_2 \leq  \frac{\eta_0}{l} \|J(\theta_0)\|_{F} \|g(\theta_0)\|_2 \leq  \frac{K\eta_0}{\sqrt l}R_0 ~. \\
         \|g(\theta_0)\|_2  &\leq R_0
\end{cases}
\end{align*}
Assume (\ref{eq:convergence-parameters:supp}) and (\ref{eq:convergence-loss:supp}) hold true for any number of training step $j$ such that $j< t$. Then, by (\ref{eq:jacobian-lip}) and (\ref{eq:convergence-loss:supp}), we have
 \begin{align*}
    \|\theta_{t+1} - \theta_t \|_2 \leq \eta \|J(\theta_t)\|_{op} \|g(\theta_t)\|_2 \leq  \frac{K\eta_0}{\sqrt l} \left(1 - \frac {\eta_0 \mins}{3}\right)^t R_0 ~.
\end{align*}

Beside, with the mean value theorem and (\ref{eq:gd&jacobian}), we have the following 
\begin{align*}
    \|g(\theta_{t+1})\|_2 &= \| g(\theta_{t+1} - g(\theta_t) + g(\theta_t))\|_2\\
    &=\|J( \theta_t^\mu) (\theta_{t+1}-\theta_t) + g(\theta_t)\|_2\\
    &= \|(I-\eta J( \theta_t^\mu)J(\theta_t)^\top ) g(\theta_t)\|_2\\
    &\leq  \|I-\eta J( \theta_t^\mu)J(\theta_t)^\top\|_{op} \|g(\theta_t)\|_2\\
    &\leq \|I-\eta J( \theta_t^\mu)J(\theta_t)^\top\|_{op} \left(1 - \frac {\eta_0 \mins}{3}\right)^t R_0
\end{align*}
where we define $ \theta_t^\mu$ as a linear interpolation between $\theta_t$ and $\theta_{t+1}$ such that $\theta_t^\mu\coloneqq \mu \theta_t + (1-\mu) \theta_{t+1}$ for some $0< \mu <1$. 

Now, we will show that with probability $1-\frac{\delta_0}{2}$,
\begin{align*}
    \|I-\eta J( \theta_t^\mu)J(\theta_t)^\top\|_{op} \leq 1 - \frac{\eta_0 \mins}{3} .
\end{align*}
Recall that $\metantk_0 \rightarrow \metaNTK$ in probability, proved by Lemma \ref{lemma:kernel-convergence}. Then, there exists $\hat l$ such that the following holds with probability at least $1-\frac{\delta_0}{5}$ for any width $l > \hat l$,
\begin{align*}
    \|\metaNTK - \metantk_0\|_F \leq \frac{\eta_0 \mins}{3}.
\end{align*}
Our assumption $\eta_0 < \frac{2}{\Lev + \lev}$ makes sure that
\begin{align*}
    \|I - \eta_0 \metaNTK\|_{op} \leq 1 - \eta_0 \mins ~.
\end{align*}
Therefore, as $l \geq (\frac{18 K^3 R_0}{\mins^2})^2$, with probability at least $1- \frac{\delta_0}{2}$ the following holds,
\begin{align*}
    &\qquad \|I - \eta J(\theta_t^\mu) J(\theta_t)^\top\|_{op} \\
    &= \|I - \eta_0 \metaNTK + \eta_0 \metaNTK - \metantk_0 +  \eta\left(J(\theta_0) J(\theta_0)^\top - J(\theta_t^\mu )J(\theta_t)^\top\right)\|_{op}\\
    &\leq \|I - \eta_0 \metaNTK\|_{op} + \eta_0 \|\metaNTK - \metantk_0\|_{op} + \eta\|J(\theta_0) J(\theta_0)^\top - J(\theta_t^\mu )J(\theta_t)^\top\|_{op}\\
    &\leq 1- \eta_0 \mins + \frac{\eta_0\mins}{3} + \eta_0 K^2 (\|\theta_t - \theta_0\|_2 + \|\theta_t^\mu - \theta_0\|_2)\\
    &\leq 1 - \eta_0 \mins + \frac{\eta_0 \mins}{3} +  \frac{6 \eta_0 K^3 R_0}{\mins\sqrt l}\\
    &\leq 1 - \frac{\eta_0 \mins}{3}
\end{align*}
where we used the equality $\frac 1 l J(\theta_0)J(\theta_0)^\top = \metantk_0$.

Hence, as we choose $\Lambda = \max\{l^*,\wt l, \hat l, \frac{18 K^3 R_0}{\mins^2})^2\}$, the following holds for any width $l > \Lambda$ with probability at least $1-\delta_0$ over random initialization
\begin{align}
    \|g(\theta_{t+1}\|_2  \leq \|I-\eta J( \theta_t^\mu)J(\theta_t)^\top\|_{op} \left(1 - \frac {\eta_0 \mins}{3}\right)^t R_0 \leq \left(1 - \frac {\eta_0 \mins}{3}\right)^{t+1} R_0,
\end{align}
which finishes the proof (\ref{eq:convergence-loss:supp}).

Finally, we prove (\ref{eq:convergence-metantk:supp}) by
\begin{align*}
    \|\metantk_0 - \metantk_t\|_F &= \frac{1}{l} \|J(\theta_0) J(\theta_0)^\top - J(\theta_t) J(\theta_t)^\top\|_F \\
    &\leq \frac{1}{l} \|J(\theta_0)\|_{op} \|J(\theta_0)^\top - J(\theta_t)^\top\|_F + \frac 1 l \|J(\theta_t)-J(\theta_0)\|_{op} \|J(\theta_t)^\top \|_F\\
    &\leq 2 K^2 \|\theta_0 - \theta_t\|_2\\
    &\leq \frac{6 K^3 R_0}{\mins \sqrt l},
\end{align*}
where we used (\ref{eq:convergence-parameters:supp}) and Lemma \ref{lemma:local-Liphschitzness}.
\end{proof}

\section{Analytical Expression of MAML Output}\label{supp:GBML-output}

In this section, we will present Corollary \ref{corr:GBML-output}. Briefly speaking, with the help of Theorem \ref{thm:global-convergence:restated}, we first show the training dynamics of MAML with over-parameterized DNNs can be described by a differential equation, which is analytically solvable. By solving this differential equation, we obtain the expression for MAML output on any training or test task. 

\textbf{Remarks.} This corollary implies for a sufficiently over-parameterized neural network, the training of MAML is \textit{determined} by the \textit{parameter initialization}, $\theta_0$. Given access to $\theta_0$, we can compute the functions $\metantk_0$ and $F_0$, and then the trained MAML output can be obtained by simple calculations, without the need for running gradient descent on $\theta_0$. This nice property enables us to perform a deeper analysis on MAML with DNNs.

\begin{corollary}[MAML Output (Corollary \ref{corr:GBML-output} Restated)]\label{corr:GBML-output}
In the setting of
Theorem \ref{thm:global-convergence}, the training dynamics of the MAML can be described by a differential equation
$$\frac{d F_t(\XXY)}{d t}=- \eta  \, \metantk_0   (F_t(\XXY) - \Y)$$
where we denote $F_t \equiv F_{\theta_t}$ and $\metantk_0 \equiv \metantk_{\theta_0}((\XXY),(\XXY))$ for convenience.

Solving this differential equation, we obtain the meta-output of MAML on training tasks at any training time as 
\begin{align}
    F_t(\XXY)=(I - e^{- \eta\metantk_0 t})\Y + e^{-\eta \metantk_0 t}F_{0}(\XXY) \,. 
\end{align}

Similarly, on arbitrary test task $\task=(\xyxy)$, the meta-output of MAML is
\begin{align}
F_t(\xxy) 
=F_{0}(\xxy)  +  \metantk_0(\xxy) \T^{\eta}_{\metantk_0}(t)
\left(\Y-F_0(\XXY)\right)
\end{align}
where $\metantk_0(\cdot)\equiv \metantk_{\theta_0}(\cdot,(\XXY))$ and $\T^{\eta}_{\metantk_0}(t)=\metantk_0^{-1}\left(I- e^{-\eta\metantk_0 t}\right)$ are shorthand notations.
\end{corollary}

\begin{proof}
For the optimization of MAML, the gradient descent on $\theta_t$ with learning rate $\eta$ can be expressed as
\begin{align}
    \theta_{t+1} &= \theta_t - \eta \nabla_{\theta_t}  \loss(\theta_t) \nonumber \\
    &=\theta_t - \frac{1}{2}\eta \nabla_{\theta_t} \|F_{\theta_t}(\XXY) - \Y\|_2^2 \nonumber \\
    &=\theta_t - \eta \nabla_{\theta_t}F_{\theta_t}(\XXY)^\top \left(F_{\theta_t}(\XXY) - \Y\right) 
\end{align}

Since the learning rate $\eta$ is sufficiently small, the \textit{discrete-time} gradient descent above can be re-written in the form of \textit{continuous-time} gradient descent (i.e., gradient flow),
\begin{align}
    \frac{d \theta_t}{d t} = - \eta \nabla_{\theta_t}F_{\theta_t}(\XXY)^\top \left(F_{\theta_t}(\XXY) - \Y\right) \nonumber
\end{align}
Then, the training dynamics of the meta-output $F_{t}(\cdot)\equiv F_{\theta_t}(\cdot)$ can be described by the following differential equation,
\begin{align}
    \frac{d F_t(\XXY)}{d t} &= \nabla_{\theta_t} F_t(\XXY) \frac{d \theta_t}{d t}\nonumber\\
    &=- \eta \nabla_{\theta_t} F_{t}(\XXY)  \nabla_{\theta_t}F_{t}(\XXY)^\top \left(F_{t}(\XXY) - \Y\right) \nonumber\\
    &=-\eta \metantk_{t} \left ( F_{t}(\XXY) - \Y \right) \label{eq:corr:diff-F}
\end{align}
where $\metantk_t = \metantk_{t}((\XXY),(\XXY))=\nabla_{\theta_t} F_{t}(\XXY)  \nabla_{\theta_t}F_{t}(\XXY)^\top$.

On the other hand, Theorem \ref{thm:global-convergence:restated} gives the following bound in (\ref{eq:convergence-metantk:supp}),
\begin{align}
        \sup_{t} \|  \metantk_0 - \metantk_t\|_F  &\leq \frac {6K^3\lss}{\mins}  l^{-\frac 1 2},
\end{align}
indicating $\metantk_t$ stays almost constant during training for sufficiently over-parameterized neural networks (i.e., large enough width $l$). Therefore, similar to \cite{lee2019wide}, we can replace $\metantk_t$ by $\metantk_0$ in (\ref{eq:corr:diff-F}), and get
\begin{align}
     \frac{d F_t(\XXY)}{d t} = -\eta \metantk_{0} \left ( F_{t}(\XXY) - \Y \right) ,
\end{align}
which is an ordinary differential equation (ODE) for the meta-output $F_t(\XXY)$ w.r.t. the training time $t$.

This ODE is analytically solvable with a unique solution. Solving it, we obtain the meta-output on training tasks at any training time $t$ as,
\begin{align}
    F_t(\XXY)=(I - e^{- \eta\metantk_0 t})\Y + e^{-\eta \metantk_0 t}F_{0}(\XXY).
\end{align}
The solution can be easily extended to any test task $\task = (\xyxy)$, and the meta-output on the test task at any training time is
\begin{align}
F_t(\xxy) =F_{0}(\xxy)  +  \metantk_0(\xxy) \T^{\eta}_{\metantk_0}(t)
\left(\Y-F_0(\XXY)\right) ,
\end{align}
where $\metantk_0(\cdot)\equiv \metantk_{\theta_0}(\cdot,(\XXY))$ and $\T^{\eta}_{\metantk_0}(t)=\metantk_0^{-1}\left(I- e^{-\eta\metantk_0 t}\right)$ are shorthand notations.
\end{proof}

\section{Gradient-Based Meta-Learning as Kernel Regression}
\label{supp:MetaNTK}
In this section, we first make an assumption on the scale of parameter initialization, then we restate Theorem \ref{thm:MetaNTK}. After that, we provide the proof for Theorem \ref{thm:MetaNTK}.

\cite{lee2019wide} shows the output of a neural network randomly initialized following (\ref{eq:recurrence}) is a zero-mean Gaussian with covariance determined by $\sigma_w$ and $\sigma_b$, the variances corresponding to the initialization of weights and biases. Hence, small values of $\sigma_w$ and $\sigma_b$ can make the outputs of randomly initialized neural networks approximately zero. We adopt the following assumption from \cite{CNTK} to simplify the expression of the kernel regression in Theorem \ref{thm:MetaNTK}.

\begin{assumption}[Small Scale of Parameter Initialization]\label{assum:small-init}
The scale of parameter initialization is sufficiently small, i.e., $\sigma_w,\sigma_b$ in (\ref{eq:recurrence}) are small enough, so that $f_0(\cdot) \simeq 0$.
\end{assumption}

Note the goal of this assumption is to make the output of the randomly initialized neural network negligible. The assumption is quite common and mild, since, in general, the outputs of randomly initialized neural networks are of small scare compared with the outputs of trained networks \cite{lee2019wide}. 

\begin{theorem}[MAML as Kernel Regression (Theorem \ref{thm:MetaNTK} Restated)]\label{thm:MetaNTK:supp}
Suppose learning rates $\eta$ and $\lambda$ are infinitesimal. As the network width $l$ approaches infinity, with high probability over random initialization of the neural net, the MAML output, (\ref{eq:F_t:main_text}), converges to a special kernel regression,
\begin{align}\label{eq:F_t-MetaNTK:supp}
F_t(\xxy)= G_\NTK^{\tau}(\xxy) +\metaNTK((\xx),(\XX)) \T^{\eta}_{\metaNTK}(t) \left(\Y-G_{\NTK}^{\tau}(\XXY)\right)
\end{align}
where $G$ is a function defined below, $\NTK$ is the neural tangent kernel (NTK) function from \cite{ntk} that can be analytically calculated without constructing any neural net, and $\metaNTK$ is a new kernel, which name as Meta Neural Kernel (MNK). The expression for $G$ is
\begin{align}
    G_\NTK^\tau(\xxy) =  \NTK(X,X')\widetilde{T}^{\lambda}_\NTK (X',\tau)   Y'.
\end{align}
where $\widetilde{T}^{\lambda}_\NTK (\cdot,\tau) \coloneqq \NTK(\cdot,\cdot)^{-1}(I-e^{-\lambda \NTK(\cdot,\cdot) \tau}) $. Besides, $G_\NTK^{\tau}(\XXY) = (G_\NTK^{\tau}(\xxyi))_{i=1}^N$.

The MNK is defined as $\metaNTK \equiv \metaNTK((\XX),(\XX)) \in \mathbb{R}^{knN \times knN}$, which is a block matrix that consists of $N \times N$ blocks of size $kn\times kn$. For $i,j \in [N]$, the $(i,j)$-th block of $\metaNTK$ is
\begin{align} \label{eq:MetaNTK_ij=kernel:supp}
    [\metaNTK]_{ij}=\SingleTaskmetaNTK((\xxi),(\xxj)) \in \mathbb R^{kn \times kn} ,
\end{align}
where $\SingleTaskmetaNTK: (\mathbb{R}^{n \times k} \times \mathbb{R}^{m\times k}) \times  (\mathbb{R}^{n \times k} \times \mathbb{R}^{m\times k}) \rightarrow \mathbb{R}^{nk \times nk}$ is a kernel function defined as
\begin{align}\label{eq:MetaNTK_ij:supp}
    \SingleTaskmetaNTK((\cdot,\ast), (\bullet, \star)) &= \NTK(\cdot,\bullet) + \NTK(\cdot,\ast)\widetilde{\T}_{\NTK}^\lambda(\ast,\tau)\NTK(\ast,\star)\widetilde{\T}_{\NTK}^\lambda(\star,\tau)^\top \NTK(\star,\bullet) \nonumber\\
 &\quad -\NTK(\cdot,\ast)\widetilde{\T}_{\NTK}^\lambda(\ast,\tau) \NTK(\ast,\bullet) - \NTK(\cdot,\star) \widetilde{\T}_{\NTK}^\lambda(\star,\tau)^\top \NTK(\star,\bullet) .
\end{align}
Here $\metaNTK((\xx),(\XX)) \in \mathbb{R}^{kn \times knN}$ in (\ref{eq:F_t-MetaNTK}) is also a block matrix, which consists of $1 \times N$ blocks of size $k n \times k n$, with the $(1,j)$-th block as follows for $j \in [N]$,
\begin{align}\label{eq:MetaNTK_1j:supp}
    [\metaNTK((\xx),(\XX))]_{1,j}= \SingleTaskmetaNTK((X,X'),(X_j,X_j')).
\end{align}
\end{theorem} 
\begin{proof}
First, (\ref{eq:F_t:main_text}) shows that the output of MAML on any test task $\task = (\xyxy)$ can be expressed as
\begin{align}\label{eq:thm:MetaNTK:F_t:finite-width}
F_t(\xxy) =F_{0}(\xxy)  + \metantk_0(\xxy)\T^{\eta}_{\metantk_0}(t)\left(\Y-F_0(\XXY)\right)    
\end{align}
Note (\ref{eq:lemma-proof:kernel-convergence:F_0}) in Appendix \ref{supp:global-convergence:kernel-convergence} shows that 
\begin{align}\label{eq:thm:MetaNTK:F_0:finite-width}
 F_0(\xxy) =f_0(\sX) + \ntk_0(\xx)\widetilde{\T}_{\ntk_0}^\lambda(\sX',\tau)\left(\sY' - f_0(\sX')\right),
\end{align}
With Assumption \ref{assum:small-init}, we can drop the terms $f_0(X)$ and $f_0(X')$ in (\ref{eq:thm:MetaNTK:F_0:finite-width}). Besides, from \cite{ntk,CNTK,lee2019wide}, we know $\lim_{l \rightarrow \infty}\ntk_0(\cdot,\star) = \NTK(\cdot,\star)$, the Neural Tangent Kernel (NTK) function, a determinisitc kernel function. Therefore, $F_0(\xxy)$ can be described by the following function as the width appraoches infinity,
\begin{align}\label{eq:supp:MNK:G}
    \lim_{l \rightarrow \infty} F_0(\xxy) =  G_\NTK^\tau(\xxy) =  \NTK(X,X')\widetilde{T}^{\lambda}_\NTK (X',\tau)   Y'.
\end{align}
where $\widetilde{T}^{\lambda}_\NTK (\cdot,\tau) \coloneqq \NTK(\cdot,\cdot)^{-1}(I-e^{-\lambda \NTK(\cdot,\cdot) \tau}) $. Besides,  $G_\NTK^{\tau}(\XXY) = (G_\NTK^{\tau}(\xxyi))_{i=1}^N$.

In addition, from Lemma \ref{lemma:kernel-convergence}, we know $\lim_{l \rightarrow \infty}\metantk_0(\cdot,\star)=\metaNTK(\cdot,\star)$. Combined this with (\ref{eq:supp:MNK:G}), we can express (\ref{eq:thm:MetaNTK:F_t:finite-width}) in the infinite width limit as
\begin{align}
F_t(\xxy)= G_\NTK^{\tau}(\xxy) +\metaNTK((\xx),(\XX)) \T^{\eta}_{\metaNTK}(t) \left(\Y-G_{\NTK}^{\tau}(\XXY)\right)
\end{align}
where $\metaNTK(\cdot,\star)$ is a kernel function that we name as Meta Neural Kernel function. The derivation of its expression shown in (\ref{eq:MetaNTK_ij=kernel:supp})-(\ref{eq:MetaNTK_1j:supp}) can be found in Appendix \ref{supp:global-convergence:lemma-proof:kernel-convergence}.
\end{proof}

\paragraph{ANIL Kernel} The above theorem derives the analytical expression of the kernel induced by MAML. Certainly that variants algorithms of MAML will induce kernels with (slightly) different expressions. A recent impactful variant of MAML is Almost-No-Inner-Loop (ANIL) \cite{raghu2019rapid}. ANIL is a simplification of MAML that retains the performance of MAML while enjoying a significant training speedup. The key idea of ANIL is to remove the inner-loop updates on the hidden layers; in other words, ANIL only update the last linear layer in the inner loop, resulting in a much smaller computation and memory cost compared with MAML. Following procedures in Appendix \ref{supp:GBML-output} and \ref{supp:MetaNTK}, one can straightforwardly derive the expression of the kernel induced by ANIL, which just replaces Eq. \eqref{eq:MetaNTK_ij:supp} (kernel function induced by MAML) by
\begin{align}\label{eq:ANIL_ij:supp}
    \SingleTaskmetaNTK((\cdot,\ast), (\bullet, \star)) &= \NTK(\cdot,\bullet) + \NNGP(\cdot,\ast)\widetilde{\T}_{\NNGP}^\lambda(\ast,\tau)\NTK(\ast,\star)\widetilde{\T}_{\NNGP}^\lambda(\star,\tau)^\top \NNGP(\star,\bullet) \nonumber\\
 &\quad -\NNGP(\cdot,\ast)\widetilde{\T}_{\NNGP}^\lambda(\ast,\tau) \NTK(\ast,\bullet) - \NTK(\cdot,\star) \widetilde{\T}_{\NNGP}^\lambda(\star,\tau)^\top \NNGP(\star,\bullet) .
\end{align}
where $\NNGP$ is the neural tangent kernel function corresponds to neural networks with frozen hidden layers (i.e., only the last linear layer is optimized by gradient descent). The appearance of $\NNGP$ directly results from the special inner-loop update rule of ANIL (i.e., only updates the last linear layer in the inner loop).

\section{More Details on Experiments} \label{supp:exp}

\paragraph{Training Data Augmentation} Following previous few-shot learning works \cite{metaOptNet,tian2020rethink}, in the training stage, we adopt data augmentation operations, including random cropping, color jittering, and random horizontal flip. 

\paragraph{Training Batch Size} For all 5-cells experiments, a batch size of 64 is used. For 8-cells experiments, we set the batch size to 40 for miniImageNet and 56 for tieredImageNet to accommodate the GPU memory of a single GPU card.

\paragraph{Dropout Rate} We use dropout rate of 0.2 and 0.1 for miniImageNet and tieredImageNet, respectively. Following DARTS \cite{DARTS}, we gradually increase the dropout rate during the training.

\paragraph{Normalization Layers} To enable efficient computation of per-sample-gradients with Opacus \cite{opacus} (it does not support BatchNorm so far), we first convert all the BatchNorm \cite{batchnorm} layers to GroupNorm \cite{groupnorm} layers with 16 number of groups in the search stage. After obtaining the cells, we train and evaluate the selected architectures with BatchNorm layers.  

\paragraph{Hyper-parameters for Computing MetaNTK} MAML kernels (defined in Definition \ref{def:metaNTK})) and ANIL kernels (defined in Eq. \eqref{eq:ANIL_ij:supp}) are used for 5-cells and 8-cells experiments, respectively. To write more concisely, We denote the product of inner loop learning rate and training time as $\lambda \tau$. An $\lambda \tau = \infty$ and a regularization coefficient of 0.001 are used for all 5-cells experiments. For 8-cells experiments, an $\lambda \tau = 1$ and a kernel regularization coefficient of $10^{-5}$ are used for miniImageNet experiments while an $\lambda \tau = \infty$ and a kernel regularization coefficient of 0.001 are used for tieredImageNet experiments.

\paragraph{Hyper-parameters for Evaluation} In the evaluation stage, we fine-tune the last layer of the learned neural net on the labelled support samples of each test task, and then evaluate its prediction accuracy on the query samples. Following the evaluation strategies of RFS \cite{tian2020rethink}, (\textit{i}) we normalize the last hidden layer output of each sample to unit norm before passing to the last layer during the evaluation; (\textit{ii}) we enlarge the set of support samples by applying data augmentation (used in the training stage) to create 5x augmented support samples for fine-tuning. We use cross-entropy loss and hinge loss for the fine-tuning, both with $\ell_2$ regularization. For cross-entropy fine-tuning, we use the Logistic Regression (LR) solver provided in scikit-learn \cite{sklearn}; for the hinge loss fine-tuning, we adopt the C-Support Vector Classification (SVC) with linear kernel provided in scikit-learn \cite{sklearn}. Notice that these the $\ell_2$ regularization in scikit-learn solvers is controlled by a regularization parameter $C = \frac{1}{\text{$\ell_2$ penalty}}$ On mini-ImageNet: (\textit{i}) in the 5-cells case, we use SVC with $C=0.2$ for 1-shot and LR with $C=0.6$ for the 5-shot experiments; (\textit{ii}) in the 8-cells case, we use SVC with $C=0.35$ for 1-shot and LR with $C=0.4$ for the 5-shot experiments. On tiered-ImageNet: (i) in the 5-cells case, we use SVC with $C=0.75$ for 1-shot and LR with $C=0.8$ for the 5-shot experiments; (\textit{ii}) in the 8-cells case, we use LR with $C=0.95$ for 1-shot and LR with $C=0.5$ for the 5-shot experiments.
\end{document}